\pdfoutput=1
\documentclass[twosode, 11pt]{article}
\usepackage{jmlr2e}
\usepackage[margin=1in]{geometry}
\usepackage{hyperref}
\usepackage{color}
\usepackage{amssymb}
\usepackage{enumitem}
\usepackage{amsfonts}
\usepackage{amsxtra}     % Use various AMS packages
\usepackage{srcltx}
\usepackage{graphicx}
\usepackage{verbatim}
\usepackage{algorithmicx}
\usepackage{algpseudocode}
\usepackage[]{algorithm2e}
\definecolor{blue}{rgb}{0,0,0.7}

\usepackage{caption}
\usepackage{subcaption}

\RestyleAlgo{boxruled}

\newcommand{\I}[1]{\mathbb{I} \left( #1 \right)}
\newcommand{\E}[1]{\mathbb{E} \left( #1 \right)}

\newcommand{\x}{\mathbf{x}}
\newcommand{\parens}[1]{\left(#1\right)}
\newcommand{\condpr}[2]{\mathbb{P}\left(#1 | #2\right)} % conditional prob
\newcommand{\condprobyx}{\mathbb{P}\left(y=1 | x\right)} % conditional prob
\newcommand{\pr}[1]{\mathbb{P} \left( #1\right)} 
\newcommand{\prsubs}[2]{\mathbb{P}_{#1} \left( #2\right)} % prob
\newcommand{\condprhat}[2]{\widehat{\mathbb{P}}\left(#1 | #2\right)} % conditional prob
\newcommand{\prhat}[1]{\widehat{\mathbb{P}} \left( #1\right)} % prob

\newcommand{\mtry}{\textit{mtry}}
\newcommand{\Strong}{\mathcal{S}}
\newcommand{\Weak}{\mathcal{W}}
\newcommand{\R}[2]{\mathcal{R}_{#1}\left( #2 \right)}
\newcommand{\N}[2]{\text{N}_{#1}\left(#2 \right)}
\newcommand{\rf}[2]{\text{RF}^{\text{#2}}\left( #1 \right)}

\usepackage{epsfig}
\usepackage{enumerate}
\usepackage{color}
\ShortHeadings{Random Forest Probability Estimation}{}
\firstpageno{1}

\begin{document}

\title{Making Sense of Random Forest Probabilities: a Kernel Perspective}

\author{\name Matthew Olson \email maolson@wharton.upenn.edu\\
\name Abraham J. Wyner \email  ajw@wharton.upenn.edu\\
       \addr Department of Statistics\\
       Wharton School, University of Pennsylvania\\
       Philadelphia, PA 19104, USA
	}
\editor{}

\maketitle

\begin{abstract}%   <- trailing '%' for backward compatibility of .sty file
A random forest is a popular tool for estimating probabilities in machine learning classification tasks.  However, the means by which this is accomplished is unprincipled: one simply counts the fraction of trees in a forest that vote for a certain class.  In this paper, we forge a connection between random forests and kernel regression.  This places random forest probability estimation on more sound statistical footing.  As part of our investigation, we develop a model for the proximity kernel and relate it to the geometry and sparsity of the estimation problem.  We also provide intuition and recommendations for tuning a random forest to improve its probability estimates.
\end{abstract}

 
% \newpage

\section{Introduction}
\label{sec:intro}

In classification tasks, one is often interested in estimating the probability that an observation falls into a given class - the conditional class probability.   These probabilities have numerous applications, including ranking, expected utility calculations, and classification with unequal costs.  The standard approach to probability estimation in many areas of science relies on logistic regression.  However, modern data sets with nonlinear or high dimensional structure, it is usually impossible to guarantee a logistic model is well-specified.  In that case, resulting probability estimates may fail to be consistent \citep{kruppa2014}.  As a result, researchers have been relying more on machine learning and other nonparametric approaches to classification that make leaner assumptions.

Random forests have become a widely used tool in ``black-box" probability estimation.  This technique has been found to be successful in diverse areas such as medicine \citep{gurm2014}, \citep{escobar2015}, ecology \citep{evans2011}, outcome forecasting in sports \citep{lock2014}, and propensity score calculations in observational studies \citep{zhao2016}, \citep{lee2010}.  First proposed in \cite{breiman2001}, a random forest  consists of a collection of randomly grown decision trees whose final prediction is an aggregation of the predictions from individual trees.  Random forests enjoy a number of properties that make them suitable in practice, such as trivially parallelizable implementations, adaptation to sparsity, and automatic variable selection, to name a few.  The reader is well-advised to consult \cite{biau2016} for an excellent review of recent research in this area.

After fitting a classification random forest to training data, it is common practice to infer conditional class probabilities for a test point by simply counting the fraction of trees in the forest that vote for a certain class.  A priori, this is an unprincipled practice: the fraction of votes of classifiers in an ensemble need not have anything to do with a class probability.  In the case when base classifiers in the ensemble are highly correlated - such as a collection of bagged tree stumps - the estimated probabilities will necessary converge to 0 or 1.

The next section contains a simulated example for which a random forest produces catastrophically poor probability estimates, yet still manages to obtain the Bayes error rate.  These observations should not be surprising.  Achieving a low misclassification error rate requires only that the classifier estimates one quantile well: the median.  As long as 251 out of 500 trees vote for the correct class, the forest will achieve a low test error rate.  Probability estimation at every quantile simultaneously is obviously a much harder problem.  What is surprising is that despite such ad hoc foundations, random forests do tend to produce good probability estimates in practice, perhaps after calibration \citep{niculescu2005}.  The goal of this paper is to put random forest probability estimates on more sound statistical footing in order to understand how a voted ensemble is able to estimate probabilities.  Moreover, we will exploit this understanding to improve the quality of these estimates in the cases where they are poor.

\subsection{Motivation}
\label{subsec:motivation}

Random forests tend to be excellent classifiers under a wide range of parameter settings \citep{berk2008}.  The same robustness, however, is not enjoyed by a forest's probability estimates.  In fact, it can sometimes be the case that such classifiers achieve the Bayes error rate while producing remarkably bad probability estimates.  The following example motivates our analysis of random forest probabilities in the rest of the paper.

We consider a very simple   model.  First, draw $n=1000$ predictors $x \in [-1,1]^{50}$ uniformly at random, and then generate class labels $y \in \{0,1\}$ according to the conditional class probability 
\begin{equation*}
\condpr{y=1}{x} =
\begin{cases}
0.3 \hspace{3mm} \text{if } x_1 < 0\\
0.7 \hspace{3mm} \text{if }x_1 \geq 0.  \\
\end{cases}
\end{equation*}
The first dimension contains all of the signal, while the remaining 49 dimensions are noise.  Note here that a simple tree stump would produce very good probability estimates.  Fit to training data, the stump would split near $x_1 = 0$, and the training data that accumulated in each daughter node would have relative proportions of $y=1$ labels in the way we would expect.  The story is quite different for a random forest.

Figure~\ref{fig:2d_hist} plots a histogram of estimated probabilities from a classification random forest under two different parameter settings.  The parameter that will concern us most in this paper is $\mtry$, the number of randomly chosen candidate predictors for each tree node - the details are spelled out in Section~\ref{sec:background}.  Since the population conditional class probability function only takes on two possible values, namely $0.3$ and $0.7$, we would ideally expect these histograms to consist of two point masses at these values.  Figure~\ref{subfig:2d_hist_mtry1} shows the results from using a random forest with $\mtry=1$, while Figure~\ref{subfig:2d_hist_mtry30} shows the estimated probabilities when $\mtry=30$ \footnote{Note that when $\mtry=1$, one randomly chosen predictor is considered at each split, so the trees in the forest are very weak.}.  When comparing these figures, the quality of probability estimates differs drastically.  When $\mtry=1$, the probabilities are pushed toward the uninformative value of $0.5$, while when $\mtry=30$, the probabilities are centered around their true values of $0.3$ and $0.7$.  However, in both cases, the random forest achieves the Bayes error rate of $0.3$!  Each random forest is able to achieve similar (optimal) classification performance in terms of test error, but very different performance in probability space.  A random forest can produce good probability estimates, but only when tuned properly.

\begin{figure}[htp]
\centering
    \begin{subfigure}[b]{0.3\textwidth}
        \includegraphics[width=\textwidth]{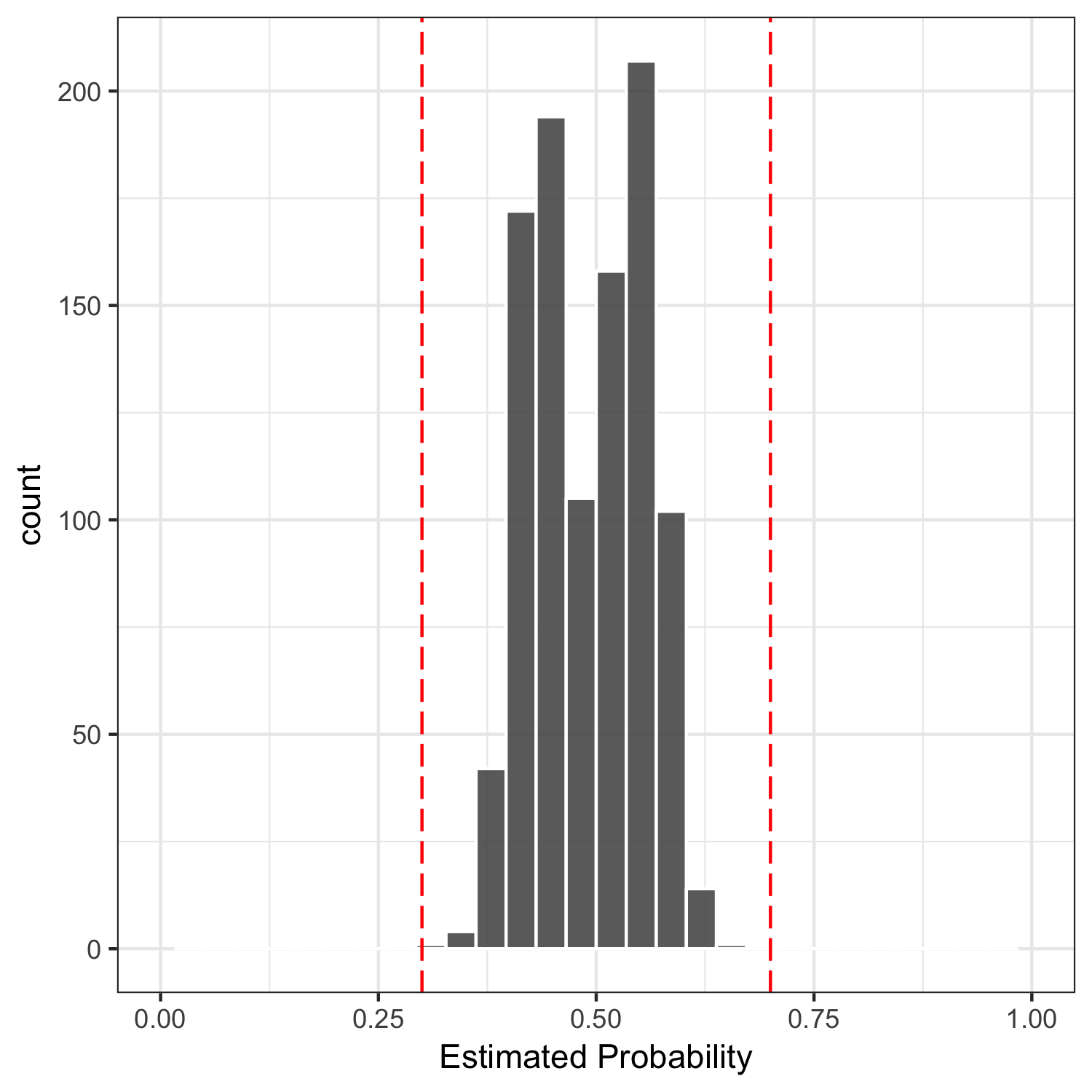}
        \caption{$\mtry=1$}
  \label{subfig:2d_hist_mtry1}
     \end{subfigure}
     \quad
     \begin{subfigure}[b]{0.3\textwidth}
        \includegraphics[width=\textwidth]{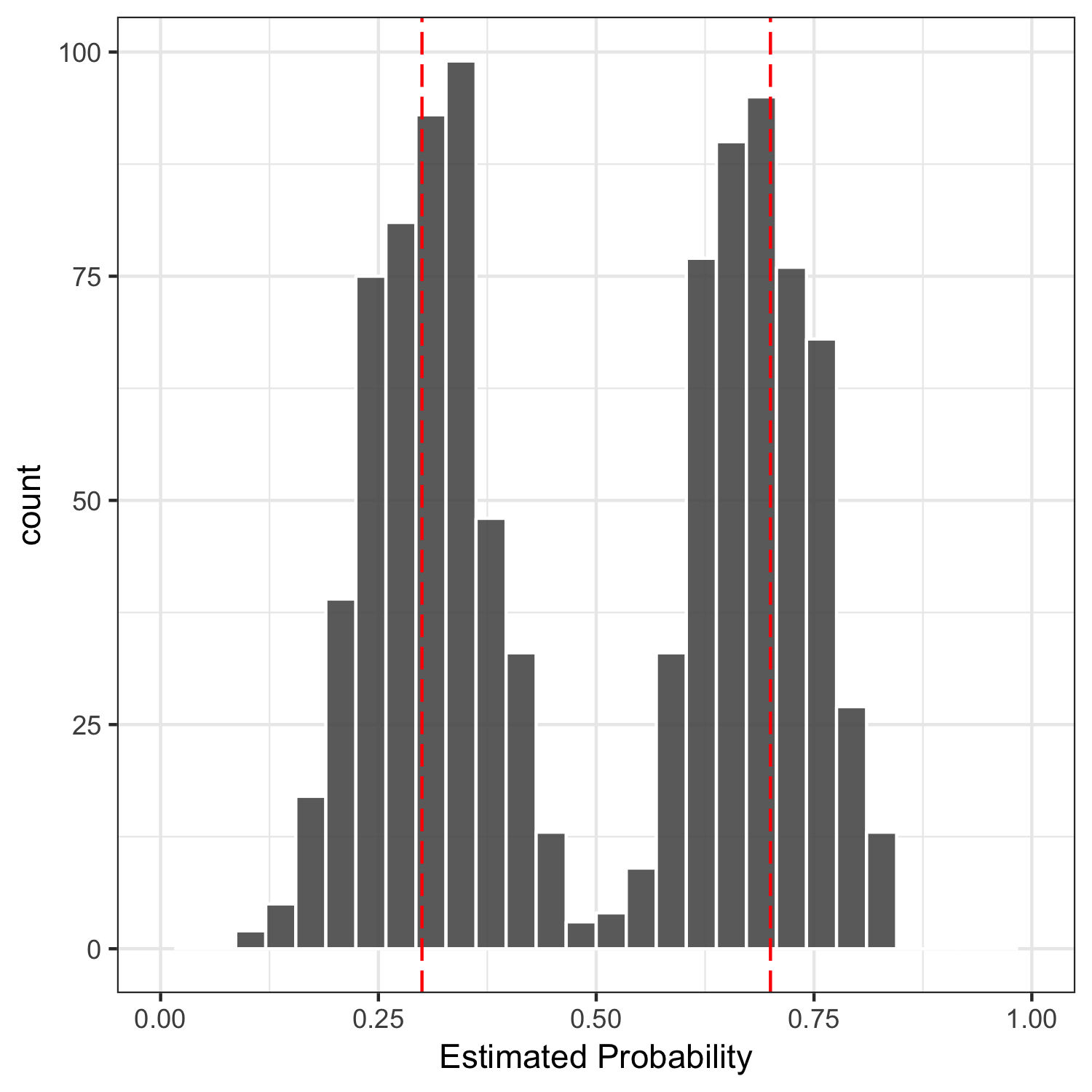}
        \caption{$\mtry=30$}
  \label{subfig:2d_hist_mtry30}
     \end{subfigure}
     \caption{Histogram of estimated probabilities produced by a random forest under two different settings of $\mtry$.  We expect these probabilities to cluster around $0.3$ and $0.7$.}
          \label{fig:2d_hist}
\end{figure}

In order to determine why this discrepancy in probability estimation quality differs, we will focus our efforts on the estimation at a single point $x_0 = (0.5, 0.5, \ldots, 0.5)$.  Specifically, we will be concerned with the points in the training set that get used to make a prediction at $x_0$, which is the same to say, the training points that appear in the same terminal nodes of the forest trees as $x_0$.  Such points are referred to in the literature as \textit{voting points}, and were first studied by \cite{lin2006}.  Figure~\ref{fig:2d_points} displays the voting points for the point $x_0$ projected in the $(x_1, x_2)$ plane.  Ideally, voting points should all lie in the half-space $x_1 \geq 0$, since these points all have the same conditional class probability as our target point $x_0$, and we would hope that a random forest would only consider ``similar" points in making a prediction.  When comparing Figures~\ref{subfig:2d_points_mtry1} and \ref{subfig:2d_points_mtry30}, it is clear that the forest with $\mtry = 30$ concentrates all of its voting points in the correct neighborhood, while the forest with $\mtry=1$ does not.  This example illustrates that the parameter $\mtry$ intuitively controls the ``tightness" of voting neighborhoods: when $\mtry$ is large, these neighborhoods concentrate more tightly.  We would like the reader to see the analogy to the bandwidth parameter in a kernel regression.  This analogy is at the heart of the paper.

\begin{figure}[htp]
\centering
    \begin{subfigure}[b]{0.3\textwidth}
        \includegraphics[width=\textwidth]{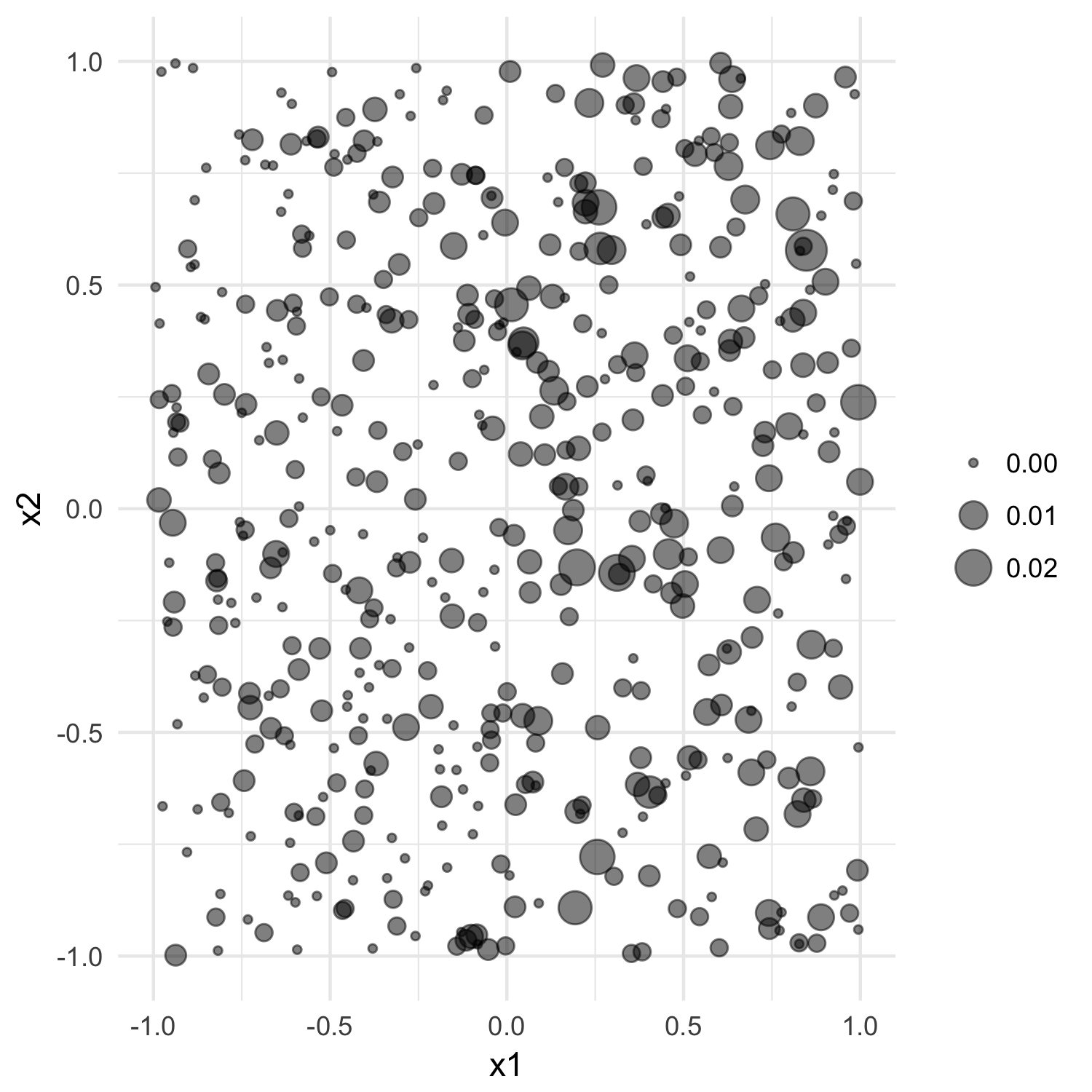}
        \caption{$\mtry=1$}
  \label{subfig:2d_points_mtry1}
     \end{subfigure}
     \quad
     \begin{subfigure}[b]{0.3\textwidth}
        \includegraphics[width=\textwidth]{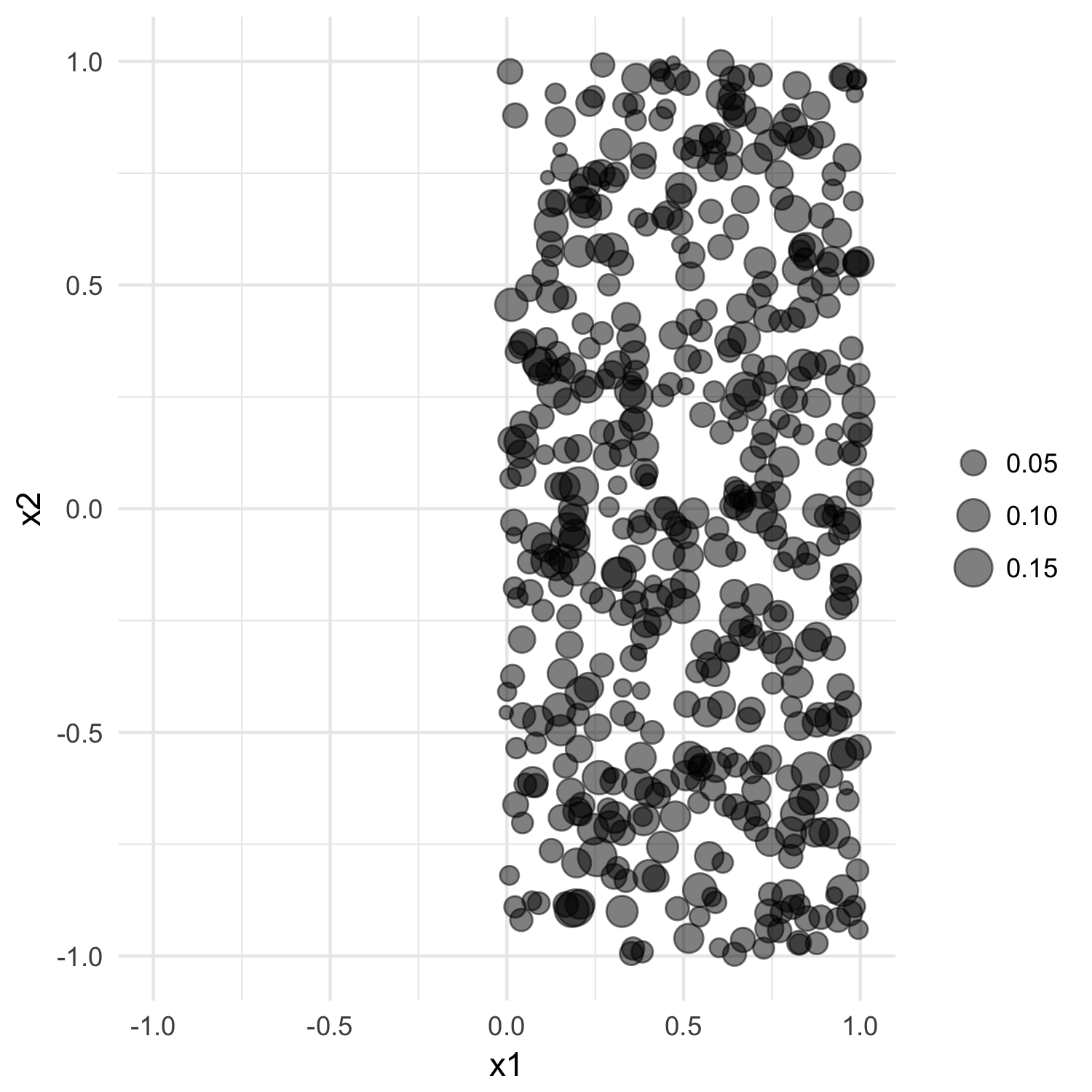}
        \caption{$\mtry=30$}
  \label{subfig:2d_points_mtry30}
     \end{subfigure}
     \caption{Plots indicating the training points contributing to the random forest's prediction at $x_0 = (0.5, 0, \ldots, 0)$ projected on the $(x_1, x_2)$ axis.  The size of the plot point indicates the fraction of trees in the forest for which each point appears in the same terminal node as the target point $x_0$. }
          \label{fig:2d_points}
\end{figure}

\subsection{Outline}
\label{subec:outline}

Our contribution is to frame random forest probability estimation in the framework of kernel regression, and to exploit this framework to understand how to use random forests to produce better probabilities.  The explicit connection between random forests and kernel methods has recently appeared in \cite{scornet2016}, although this connection is implicit in earlier works of Leo Breiman \citep{breiman2000}, \citep{breiman2004}.  This work shows that regression random forests can be viewed as close cousins of kernel regression, in which the kernel function is the \textit{proximity function}, which will be discussed in more depth later in the paper.  We build on this work by studying the shape of the random forest kernel, especially as it relates to the forest's parameter settings.  As a key tool to developing intuition about the behavior of this kernel, we develop an analytical approximation to the kernel of a simplified model of a random forest studied in other parts of the literature.  This model allows us to bridge intuition about bandwidth selection in Nadaraya-Watson type kernel estimation to understand the role of parameter tuning in random forest probability estimation.

We will begin in Section~\ref{sec:background} by providing more formal background on random forest probability estimation, as well as notation which will facilitate our discussion throughout the paper.  Next, we will introduce the concept of the \textit{proximity function} in Section~\ref{sec:prox_probs}, and we will relate random forests to kernel regression methods.  This not only allows us to view random forest probabilities in a more principled way, but also motivates discussion of the random forest proximity function as a fundamental quantity of interest in probability estimation.  In order to better understand the shape of the random forest kernel and its connection to tuning parameters, we develop a kernel approximation to a simplified model of a random forest in Section~\ref{sec:kernel_intuition}.  This will allow us to bridge our intuition about kernel regression to random forest probability estimation.  We will then confirm the intuition developed from our simple model in Section~\ref{sec:empirical_kernel}, and will consider more extensive simulation experiments in Section~\ref{sec:prob_comp}.

\section{Background}
\label{sec:background}

In this section, we will establish mathematical notation to facilitate of discussion in the rest of the paper, as well as provide background on the existing literature in probability estimation in random forests.

\subsection{Setup and Notation}
\label{subsec:notation}

In the standard set-up for binary classification problems, we observe $n$ pairs of training data points $(x_1, y_1), \ldots, (x_n, y_n)$ with $x_i \in \mathcal{X} \subset \mathbb{R}^{p}$ and $y_i \in \{0, 1\}$, and the goal is to learn a mapping from $x_i$ to $y_i$.  To make the problem more statistically tractable, it is often assumed that these pairs are independent and identically distributed, and that $x$ and $y$ are related according to an unknown conditional class probability function $\condpr{y=1}{x}$.  The goal of the analyst in classification is simply to discriminate whether $\condpr{y=1}{x} \geq 0.5$ to predict the class of a new test point $x$.  The related problem of directly estimating the probability of class membership $\condpr{y=1}{x}$ is much more difficult.  This is the problem we consider in this paper in the context of random forests.

Recall that a random forest for classification consists of a collection of $T$ un-pruned decision trees, where each tree is grown on a bootstrap sample of the data, and the split variable at each node of the tree is chosen to be the best among a subset of size $\mtry$ randomly chosen predictors.  The randomness introduced in each tree is governed by a random variable $\theta \in \Theta$ (i.e. the bootstrap sample that was chosen, as well as the candidate subset of random predictors at each node of the tree).  The parameter $\theta$ will serve us primarily as a way to index the trees in the forest.  Each decision tree partitions the input space into hyper-rectangles formed by the terminal nodes.  We will denote the terminal node in the tree generated by $\theta \in \Theta$ to which a point $x$ belongs by $\R{\theta}{x}$, and we will denote the number of sample points in this node by $\N{\theta}{x}$.

To simplify notation, we will assume that the bootstrap step is not used in the tree-growing process.  With this caveat, we can define the prediction for a single tree at a single point $x_0$ by
\begin{equation*}
f\parens{\theta, x_0} = \sum^{n}_{i=1} \frac{\I{x_i \in \R{\theta}{x_0}} y_i}{\N{\theta}{ x_0}}.
\end{equation*}
In words, to make a prediction at a point $x_0$, we simply determine which terminal node that point belongs to, and then return the fraction of training points in that node for which $y=1$.  A random forest is constructed from a selection of independent random draws $\theta_1, \ldots, \theta_T$ and the associated trees $f\parens{\theta_1, \cdot}, \ldots, f\parens{\theta_T, \cdot}$.  We will consider two types of random forests for probability estimation, \textit{classification forests}, denoted by $\rf{\cdot}{class}$ and \textit{regression forests}, denoted by $\rf{\cdot}{reg}$.  We can define these as
\begin{align*}
\rf{x_0}{reg} &=  \frac{1}{T} \sum^T_{t=1} f\parens{\theta_t, x_0} \\
\rf{x_0}{class} &= \frac{1}{T} \sum^T_{t=1} \text{round} \parens{f\parens{\theta_t, x_0}}. 
\end{align*}

In the case of a classification random forest, we estimate probabilities simply by making a class prediction for each tree $\text{round} \parens{f\parens{\theta_t, x_0}}$, and counting the fraction of trees that vote for a certain class.  In practice, classification forest trees are often grown to a terminal node size of one.  In later sections, we will be concerned with the extent to which $\rf{x_0}{class}$ and $\rf{x_0}{reg}$ approximate the conditional class probability $\condpr{y=1}{x_0}$.  It is not apriori obvious that counting the fraction of trees that vote for a certain class, as in a classification random forest, is principled way to estimate this quantity.

\subsection{Literature on Probability Estimation}
\label{subsec:lit}

  There has been relatively little literature on probability estimation in random forests.  It is known that classification random forests probabilities are typically uncalibrated, but produce among the best estimates among machine learning classifiers after calibration \citep{niculescu2005}.  Other research has investigated the usefulness of correcting probabilities in random forests using Laplace and m-estimates at the nodes \citep{bostrom2007}.  There also exists limited empirical evidence comparing the efficacy of regression and classification random forest probabilities in a number of simulated and real data settings \citep{li2013}.

  Recently, researchers have argued that regression random forests are appropriate for probability estimation \citep{malley2012}, \citep{kruppa2014}.  This claim follows from recent work showing that regression random forests are consistent in a number of settings \citep{scornet2015}.  Essentially, the argument is that if one passes the binary random variable $Y$ in a regression forest, then consistency in conditional class probability follows since $\E{Y | X} = \condpr{Y=1}{X}$.  We contend that this argument is too simplistic to explain how to obtain useful probability estimates in practice.  For one, random forest probabilities are not always good, even if this theory says otherwise \footnote{The example from Section~\ref{subsec:motivation} has the same qualitative outcome if we use a regression random forest instead of a classification random forest, even with a very large training set.}. Secondly, one must appreciate that classification and regression random forests are extremely similar, except for the default parameter settings and aggregation method.  Indeed, we argue in the appendix that in the binary outcome case, the splitting criteria for regression trees is exactly the same as that for classification trees.  We assert that performance differences in probability estimation simply result from differences in default parameter settings: regression random forests have a larger default setting for $\mtry$, and in many cases this can account for any discrepancy in probability estimation performance.  We will argue in the rest of the paper that a kernel perspective on probability estimation is more appropriate since it emphasizes that tuning parameters are critical in practice.

\section{Probabilities from the Proximity Functions}
\label{sec:prox_probs}

In this section we will introduce the concept of the proximity function and relate it to the task of probability estimation.  We will begin by motivating the proximity function as a natural distance measure induced by a random forest, and we will briefly describe some of its traditional uses.  Next, we will argue that the probability estimates produced through voting resemble those of kernel regression, where the kernel is given by the proximity function.

\subsection{The Proximity Function}
\label{subsec:prox_function}

We will begin by defining a natural distance metric induced by a random forest, the \textit{proximity function}.
\begin{definition}[The Proximity Function]
\label{def:prox}
Grow a random forest according to $\theta_1, \ldots, \theta_T$ on a training set $(x_1, y_1), \ldots, (x_n, y_n)$.  The \textit{proximity function} is a mapping $K_T : \mathcal{X} \times \mathcal{X} \rightarrow [0,1]$ where $K_T(x,z) = \frac{1}{T} \sum^T_{t=1} \I{x \in \R{\theta_t}{z}}$.
\end{definition}
In words, the proximity function simply measures the fraction of trees in a forest for which two test points appear in the same terminal node.  One can view this quantity as a natural notion of similarity between two points: the more times these two points appear in the same terminal node of a tree, the more similar they are.  In the limit of an infinite number of trees, it is also natural to consider a population version of the proximity function given by $K(x,z) = \prsubs{\theta}{z \in \R{z}{\theta_t}}$.  It is easy to argue from the law of large numbers that $K_T (x,z) \rightarrow K(x,z)$ as $T \rightarrow \infty$ almost surely \citep{breiman2000}.  Since trees generate partitions of the input space, one also has the interpretation that the proximity function gives the probability that two points are in the the same cell of a randomly chosen partition.  The notation $K_T$ and $K$ is also meant to be suggestive: both are self-adjoint, positive definite kernels that are bounded above by one \citep{breiman2000}.

The most common use of the proximity function is for clustering training data.  To this end, one can define the \textit{proximity matrix} $P \in \mathbb{R}^{n \times n}$ where $P_{i,j} = K_T \left( x_i, x_j\right)$.  In other words, the proximity matrix contains all of the pairwise similarities among the training data.  By the kernel properties of the proximity function, one can easily argue that the matrix $D \in \mathbb{R}^{n \times n}$, $D_{i,j} = 1 - P_{i,j}$ defines a Euclidean distance matrix.  One can then appeal to multidimensional scaling to a find a lower dimensional representation of the data that approximately respects the similarity measure induced from the proximity function.  In addition to clustering, the proximity function has an number of other creative applications including missing data imputation and outlier detection.  See \cite{berk2008} for more discussion.

\subsection{Proximity Probabilities}
\label{subsec:prox_probs}

Given our motivation of the proximity function as a measure of similarity between points, it is reasonable to explore how we might use this function to create probability estimates.  To this end, let us first recall the form of Nadaraya-Watson kernel-weighted regression.  For our purposes, a kernel function $K : \mathcal{X} \times \mathcal{X} \rightarrow \mathbb{R}^{+}$ is a non-negative function such that $K \left(x, z \right)$ is large when $x$ and $z$ are close in some sense.  In practice, one often considers parameterized families, such as the Gaussian family $K_\lambda \left(x, z \right) = \exp{ \parens{\frac{-||x-z||^2_2}{2 \lambda }}}$.  Here, $\lambda$ is referred to as the bandwidth of the kernel, and controls the bias-variance trade-off in estimation by changing the size of the local neighborhood.  A kernel regression estimate of conditional class probability is then given by
\begin{equation*}
\condpr{y=1}{x_0} = \frac{\sum^N_{i=1} K \parens{x_0, x_i} y_i}{ \sum^N_{i=1} K \parens{x_0, x_i} }.
\end{equation*}
If we let $w_i(x_0) = \frac{K \parens{x_0, x_i}}{\sum^N_{i=1} K \parens{x_0, x_i}}$, we can simply think of a kernel estimate as a weighted combination of training data labels, with weights proportional to the similarity between a given training point and the target points $x_0$, i.e. $\condpr{y=1}{x_0} = \sum^n_{i=1} w_i(x_0)  y_i$.

We can also make a simple argument that kernel regression is an intuitively appealing method for estimating conditional class probabilities.  Suppose we estimate marginal and conditional probabilities in the following manner: $\prhat{y=1} = \frac{n_1}{n}$, $\prhat{x_0} = \frac{1}{n} \sum^n_{i=1} K_\lambda \parens{x_0, x_i}$,  $\condprhat{x_0}{y=1} =  \frac{1}{n_1}\sum_{i: y_i = 1} K_\lambda \parens{x_0, x_i}$, where $n_1$ is the number of training points in the data set for which $y_i = 1$.  Notice that the last two expressions are just kernel density estimates of $\condpr{x_0}{y=1}$ and $\pr{x_0}$, respectively.  We can then write an estimate for the conditional class probability using Bayes rule, plugging in kernel density estimates for the appropriate quantities as follows:

\begin{align*}
\condpr{y=1}{x_0} &= \frac{\pr{y=1} \condpr{x_0}{y=1}}{\pr{x_0}} \\
&\approx \frac{\prhat{y=1} \condprhat{x_0}{y=1}}{\prhat{x_0}} \\
&= \frac{\frac{n_1}{n} \frac{1}{n_1}\sum_{i: y_i = 1} K_\lambda \parens{x_0, x_i}}{ \frac{1}{n} \sum^n_{i=1} K_\lambda \parens{x_0, x_i}} \\
&= \frac{\sum^n_{i=1} K_\lambda \parens{x_0, x_i} y_i}{\sum^n_{i=1} K_\lambda \parens{x_0, x_i}}.
\end{align*} 

We will now make a connection between the probabilities generated from voting and the proximity function.  Note that the connection between regression random forests and kernel regression was pointed out in \cite{scornet2016}.  Our focus here is to cast probability estimation in this light in order to explain the role of random forest parameters in controlling the quality of probability estimates, and we believe it is most natural to do this in the kernel regression setting.  In later sections, we will devote effort to explaining how tuning parameters such as $\mtry$ and the number of nodes in each tree act as bandwidth parameters for this estimate.

Recall from Section~\ref{sec:background} that a random forest estimate of probabilities is given by $\condpr{y=1}{x_0} = \frac{1}{T}\sum^{T}_{t=1} f\parens{x_0, \theta_t}$.  We can plug in the definition of $f \parens{x_0, \theta_t}$ to re-write this estimate in a more enlightening form:

\begin{align*}
\label{eq:prox_prob}
\condpr{y=1}{x_0} &= \frac{1}{T}\sum^{T}_{t=1} \sum^{n}_{i=1} \frac{\I{x_i \in \R{\theta_t}{x_0}} y_i}{\N{\theta_t}{x_0}} \\
&= \sum^{n}_{i=1} y_i \left( \frac{1}{T} \sum^{T}_{t=1}\frac{\I{x_i \in \R{\theta_t}{x_0}}}{\N{\theta_t}{x_0}} \right)  \\
& \approx  \sum^{n}_{i=1} y_i \left( \frac{\sum^{T}_{t=1} \I{x_i \in \R{\theta_t}{x_0}}}{\sum^{T}_{t=1} \N{\theta_t}{x_0}} \right) \\
& = \sum^{n}_{i=1} \frac{K(x_0, x_i) y_i}{\sum^n_{i=1} K(x_0, x_i)}.
\end{align*}
We will refer to the quantity in the last step as the \textit{proximity probability} estimator of conditional class probability.
\begin{definition}[Proximity Probabilities]
\label{def:prox}
The proximity probability estimate of conditional class probability at a test point $x_0$ is defined to be
\begin{equation*}
\rf{x_0}{prox} \equiv \sum^{n}_{i=1} \frac{K(x_0, x_i) y_i}{\sum^n_{i=1} K(x_0, x_i)}.
\end{equation*}
\end{definition}
In the third step, we use the approximation that $\frac{1}{T} \sum^{T}_{t=1} \frac{\text{a}_t}{\text{b}_t} \approx \frac{ \frac{1}{T} \sum^{T}_{t=1}\text{a}_t}{\frac{1}{T} \sum^{T}_{t=1}\text{b}_t}$ for any positive sequence of numbers $a_t$ and $b_t$ \footnote{We also use the fact that 
\begin{align*}  \frac{1}{T}\sum^{T}_{t=1} \N{\theta_t}{x_0} &= \frac{1}{T} \sum^{T}_{t=1} \sum^{n}_{i=1} \I{x_i \in \R{\theta_t}{x_0}} \\  
&= \sum^{n}_{i=1} \frac{1}{T} \sum^{T}_{t=1} \I{x_i \in \R{\theta_t}{x_0}} \\
&= \sum^{n}_{i=1} K \parens{x_0, x_i}
\end{align*}.}.  It is interesting to consider this approximation in the context of positive random variables $X$ and $Y$: the approximation is merely assuming that $\E{\frac{X}{Y}} \approx \frac{\E{X}}{\E{Y}}$.  Note that we simply use this approximation as intuition for motivating a different way of aggregating probabilities.  In particular we make no claims about the quality of this approximation.  When this is the case, $\rf{x_0}{prox}$ will be similar to the estimated probability estimated in the usual way.

\subsection{Intuition}
\label{subsec:prox_intuition}

We will close this section with a few final thoughts on the similarities and differences between the different ways of using a random forest for probability estimation.  First, it will be helpful to define a quantity which gives a conditional class probability estimate for a point $x_0$ for a given tree generated by $\theta_t$: $p_t\left( x_0 \right) = \sum^{n}_{i=1} \frac{\I{x_i \in \R{\theta_t}{x_0}} y_i}{\N{\theta_t}{x_0}}$.  With this notation in hand, we can give three expression for probability estimates generated by a regression, classification, and proximity random forest:
\begin{align*}
\rf{x_0}{reg} &= \frac{1}{T} \sum^{T}_{t=1} p_t\left( x_0 \right) \\
\rf{x_0}{prox} &= \sum^{T}_{t=1} \frac{\N{\theta_t}{x_0}}{\sum^{T}_{t^{'}=1} \N{\theta_{t^{'}}}{x_0}} p_t\left( x_0 \right)  \\
\rf{x_0}{class} &= \frac{1}{T} \sum^{T}_{t=1} \text{round}\left( p_t\left( x_0 \right) \right).
\end{align*}
First, one should note that in classification setting, the typical default is to grow a tree to node purity, so $\rf{x_0}{prox}$ agrees with $\rf{x_0}{class}$ in this case.  In the case where nodes are not grown to purity, there is an interesting difference between the way regression random forest probabilities and proximity probabilities differ.  In particular, regression random forests average node probabilities across trees with equal weights, while the proximity probability averages probabilities with weights proportional to the number of points in each node.  Stated in another way, $\rf{x_0}{prox}$ forms a probability estimate by collecting all points in the training set (with multiplicity) that share a terminal node with the target point, and taking a grand mean.  If one really does believe that $p_t\left( x_0 \right)$ is an estimate of a conditional class probability, this might be a reasonable strategy when some nodes have small counts.  The form of $\rf{x_0}{class}$ seems to suggest that it is unlikely that $p_t\left( x_0 \right)$ do in fact estimate probabilities: if we consider $p_t\left( x_0 \right)$ to be a random variable with expectation $\condpr{y=1}{x_0}$ and very small variance, $\rf{x_0}{class}$ would tend to produce estimates that are all close to 0 or 1.

\section{Kernel Intuition}
\label{sec:kernel_intuition}

We argued in the previous section that random forest probabilities can be fruitfully viewed from a kernel regression point of view.  Central to kernel regression is the shape of the kernel and associated bandwidth parameter.  Unfortunately, the mathematics of the original random forest algorithm forest are complicated and make analytical expressions for the kernel intractable.  In this section, we derive the kernel in a simplified setting considered in \citet{breiman2000}, \citet{breiman2004}, \citet{biau2012}.  In particular, we demonstrate that a weighted Laplacian kernel is a good analogy to a random forest, and we leverage this analogy to provide intuition for how the parameter choice in a random forest informs the quality of its probability estimates.

\subsection{A Naive Model of a Random Forest Kernel}
\label{subsec:naive_model}

While the mechanics are quite simple, the mathematics of a random forest make analytical expressions for quantities of interest intractable.  As a result, the literature tends to focus on stylized models that make analysis more amenable.  In this vein, we construct a simplified model of a random forest by borrowing elements from models considered in \citet{breiman2000}, \citet{breiman2004}, \citet{biau2012}.  Closed for expressions for other random forest models have been considered elsewhere, but our model contains richer elements that more closely capture adaptive splitting \citep{scornet2016}.  As we will demonstrate, adaptive splitting is key to constructing a kernel which adapts to the shape and sparsity of the conditional class probability function.

The simplified model we will consider is given in Algorithm~\ref{algo:simple_rf}.  To begin, we assume that the domain of the predictor variables is $\mathcal{X} = [0,1]^p$.  Of the $p$ predictors, $S$ of them are related to the response variable, and we call such predictors \textit{strong}, and the remaining $W$ predictors are unrelated to the response, and we call these weak.  One may equivalently call strong variables ``signal" variables, and weak variables ``noise" variables.  We fix a number $M$ of tree nodes before tree growing begins.  At each stage of the growing process, we select a node at random to split on, and randomly select $\mtry$ predictors.  Among the $\mtry$ variables chosen, we split only on the strong ones with equal probability (unless only weak variables are chosen).  This mechanism mimics adaptivity in the forest, and one can indeed verify in simulated examples that the predictors which are related to the response do indeed tend to be split on with higher probability \citep{biau2012}.  Finally, the chosen predictor is split on uniformly at random.

%\begin{algorithm}[htp]
%%\newcommand{\thealgorithm}{\arabic{algorithm}}
%\caption{Simplified Random Forest Tree}
%\begin{algorithmic}
%\State 1. The domain $\mathcal{X}$ is the unit cube $[0,1]^p$.  There are $S$ strong variables and $W$ weak \\ \indent \indent variables, where $S + W = p$.  Fix a number of tree nodes $M$.
%\State 2. For $m = 1:M$:
%\State \indent (a) Select a terminal node $\mathfrak{t} \in \{\mathfrak{t}_1, \ldots, \mathfrak{t}_{m-1}\}$ uniformly at random.
%\State\indent  (b) Select $\mtry$ predictors uniformly at random.  $S_*$ of these predictors will be strong, \\ 
%\indent \indent \indent and $W_*$ will be weak.
%\State\indent  (c) Among the $S_*$ strong predictors, select a strong predictor $x_*$ with probability $1/S_*$.
%\State\indent  (d) Choose a split value for $x_*$ uniformly at random to create a new node $\mathfrak{t}_m$.
%\end{algorithmic}
%\label{algo:simple_rf}
%\end{algorithm}

\begin{algorithm}[htp!]
\caption[Simplified Random Forest Tree]{Simplified Random Forest Tree}
\begin{algorithmic}
\State 1. Specify the number of leafs $M$; initialize $\texttt{leafs} = \{\mathfrak{t}_{root}\}$.
\State 2. For $m = 1:M$:
\State \indent (a) Select a terminal node $\mathfrak{t} \in \texttt{leafs}$ uniformly at random.
\State\indent  (b) Split $\mathfrak{t}$ into daughter nodes $\mathfrak{t}_{L}$, $\mathfrak{t}_{R}$
\State\indent\indent (i) Choose $\mtry$ predictors at random $\mathcal{F} \subseteq \{1, \ldots, p\}$
\State\indent\indent (ii) Select split variable uniformly among the $S_{*}$ \\ \indent \indent \indent \hspace{5mm} signal variables in $\mathcal{F}$
\State\indent\indent (iii) Choose split point uniformly at random
\State \indent (c) Replace $\mathfrak{t}$ with $\mathfrak{t}_{L}$ and $\mathfrak{t}_{R}$ in \texttt{leafs}
\end{algorithmic}
\label{algo:simple_rf}
\end{algorithm}

Note that the above algorithm gives a probability model for a fixed tree in the forest.  We are interested in the probability that such a randomly drawn tree contains two fixed points in one of its terminal nodes, which is precisely the interpretation for a random forest proximity kernel.  In this setting, we can derive an approximate expression for the proximity kernel, which is given in the following proposition.

\begin{proposition}
\label{prop:kernel}
Suppose a random forest is grown to a size of $M$ nodes.  Under the above setting, the proximity function $K\parens{0, x}$ has the following approximation:
\begin{equation*}
K(0,x) \approx \exp{ \left\lbrace -\log{M}  \left(  p_{\mathcal{S}} \sum_{s \in \Strong} x_s  +  p_{\mathcal{W}} \sum_{w \in \Weak} x_w \right) \right\rbrace}
\end{equation*}
where 
\begin{align*}
p_{\mathcal{S}} &=  \sum^{S \wedge \mtry}_{k=1} \frac{{S-1 \choose k-1} {W \choose \mtry - k+1} }{k {p \choose \mtry}} \\
p_{\mathcal{W}} &= \frac{1-S p_{\mathcal{S}}}{W}.
\end{align*}
\end{proposition}

\begin{proof}
(See Appendix~\ref{sec:approximation})
\end{proof}

It is clear that even in this simple setting, the kernel will clearly not be translation invariant.  However, one can make the further approximation, as in \cite{scornet2016}, that $K\parens{x,z} \approx K\parens{0, |x-z|}$.

In a 2000 paper, Breiman derived an approximation to the proximity function in a simplified model in which predictors were selected uniformly at random at each stage of the growing process, with no distinction between ``strong'' and ``weak" variables \citep{breiman2000}.  The form of this kernel was correspondingly more simple, and took the form of a symmetric Laplacian density: $\exp{\parens{ -\frac{\log{M}}{p} ||x-z||_1}}$.  The advantage of our formulation is that it is evident that variables receive different importance weights in the proximity function.  We will demonstrate in Section~\ref{sec:empirical_kernel} that the proximity function does tend to adapt to the shape of the underlying conditional class probability function, especially in regards to sparsity.  In our model, the mechanism by which this adaptation takes place is through weights $p_s$ and $p_w$ from Proposition~\ref{prop:kernel}.  In particular, the kernel tends to be more narrow in the direction of ``strong" variables and flatter in the direction of ``weak" variables.  Correspondingly, our model captures this phenomenon by assigning higher weights to signal variables than weak variables, and the relative widths are controlled by the $\mtry$ parameter.  It is also possible to change the shape of the kernel by assigning different weights to different strong variables, but makes the analysis a bit more tedious.

\subsection{Naive Kernel Parameters}
\label{subsec:naive_model}

In this section, we will discuss how the value of $\mtry$ and the number of tree nodes affects the shape of the kernel derived in the previous section.  The values of $p_s$ and $p_w$ are visualized as a function of $\mtry$ in Figure~\ref{fig:naive_weights}.  We fix the number of strong variables as $S=5$, and  plot the weights for varying numbers of noise variables $W$.  Notice first that for any value of $W$, the weight for strong variables $p_S$ is a strictly increasing function of $\mtry$ and the weight for weak variables $p_w$ is a strictly decreasing function of $\mtry$.  Furthermore, for large enough values of $\mtry$, one can prove the weights for strong variables asymptote to $1/S$, and the weights for weak values converge to zero.  These qualitative findings accord with what we would expect in a random forest.  When $\mtry$ is large, the selection of candidate predictors likely contains a large number of signal variables, and such variables are preferred by tree splitting criteria since they lead to more pure daughter nodes.  The consequence of a larger weight on such variables is that the kernel is more narrow in directions of signal, which again accords with what we will empirically demonstrate in Section~\ref{sec:empirical_kernel}.  Furthermore, we can see that the value of $\mtry$ needed to filter out noise variables increases with the number of weak variables.

\begin{figure}[htp]
\centering
    \begin{subfigure}[b]{0.3\textwidth}
        \includegraphics[width=\textwidth]{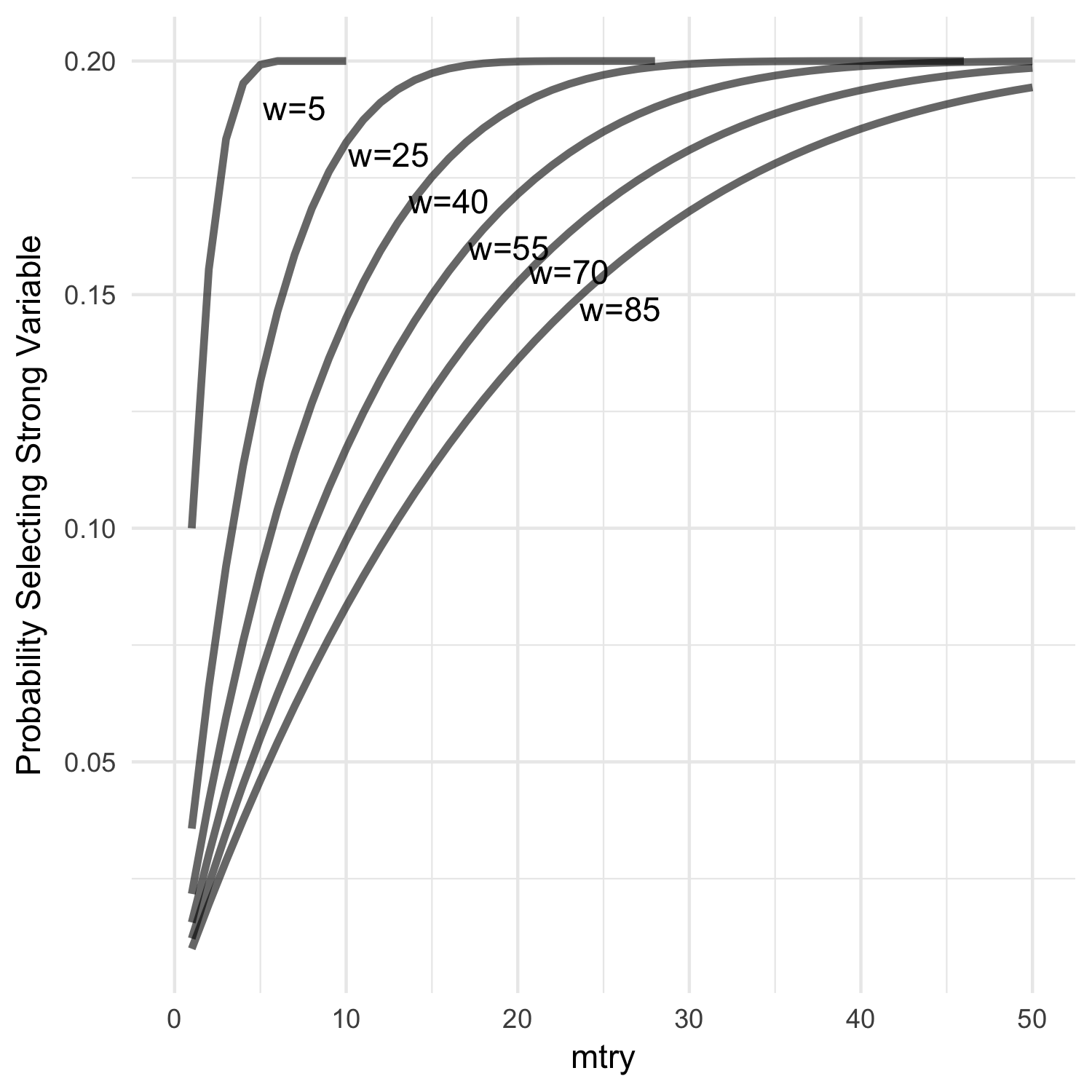}
        \caption{}
	\label{subfig:ps}
     \end{subfigure}
     \quad
     \begin{subfigure}[b]{0.3\textwidth}
        \includegraphics[width=\textwidth]{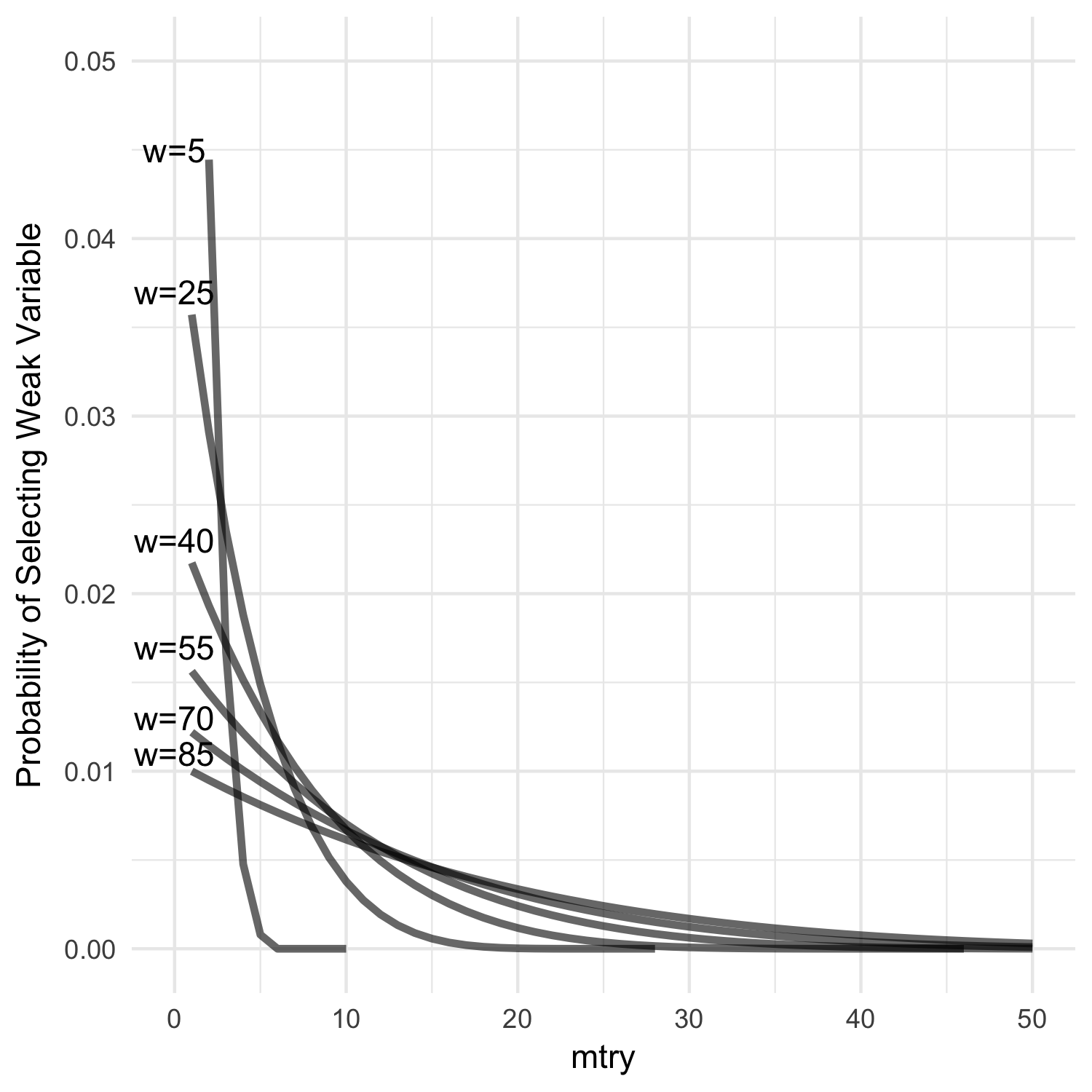}
        \caption{}
	\label{subfig:pw}
     \end{subfigure}
     \caption{The plots above show the probability of selecting strong and weak variables as a function of $\mtry$ in the naive model.  In both figures, the number of strong variables is fixed at 5, while the number of weak variable ranges from 5 to 85.}
          \label{fig:naive_weights}
\end{figure}

The number of terminal nodes also affects the shape of the proximity function derived in the previous section.  In particular, a larger number of nodes $M$ shrinks the diameter of the kernel by a multiplicative factor of $\log{M}$, making it more concentrated.  In our model, the number of nodes $M$ is a predetermined parameter, but in practice it depends on a number of characteristics of the data, such as the size of the data set and the value of $\mtry$.  For a data set of fixed size $n$, the number of tree nodes tends to decrease with larger values of $\mtry$.  The intuition is simple: an increase in $\mtry$ increases the number of signal variables appearing at each split, producing more shallow trees.  Figure~\ref{fig:terminal_nodes} shows a plot of the average number of nodes in a random forest tree for a simulated example as a function of $\mtry$, which confirms our intuition.

\begin{figure}[htp]
\centering
        \includegraphics[width=0.3\textwidth]{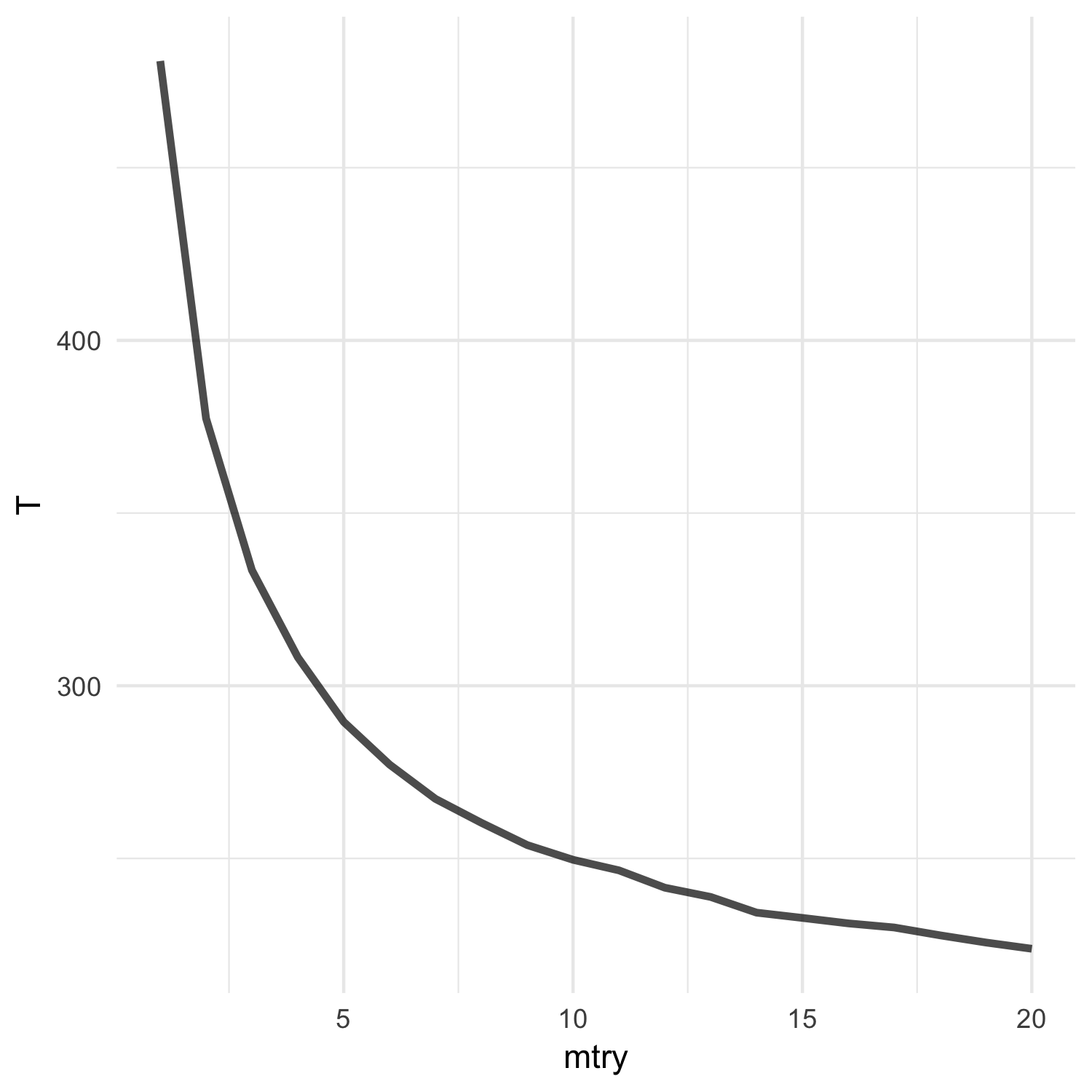}
        \caption{The number of terminal nodes as a function of $\mtry$ for a simulated example.}
	\label{fig:terminal_nodes}
\end{figure}

We can summarize the main conclusions from our simplified model as follows:

\begin{itemize}
\item[(i)] The value of $\mtry$ controls the relative width of the kernel in signal and noise dimensions.
\item[(ii)] An increase in the number of nodes in the tree shrinks the kernel equally in all directions.
\item[(iii)] Larger values of $\mtry$ are needed to sufficiently ``zero out" the weight of noise variables in high dimensions.
\end{itemize}

\subsection{Laplace Kernel Regression}
\label{subsec:laplace}

Given the form of the kernel we derived in Section~\ref{subsec:naive_model} and the implications for its parameters explored in Section~\ref{subsec:naive_model}, we now develop some intuition for a optimal choices for kernel regression with a weighted Laplace kernel.  The idea here is to explore the role of kernel weights in producing good probability estimates, and then to tie these weights to random forest parameters.  In kernel regression, one typically selects a kernel before looking at the data.  We would like to demonstrate that a kernel that adapts to the data is much more powerful, and is in fact essential in higher dimensions.

We will consider here kernels of the form 
\begin{equation*}
K_{\lambda, w} (x,z) = \exp\{ -\lambda \sum^p_{j=1} w_j |x_j - z_j| \}
\end{equation*}
where $w_i$ are non-negative weights summing to one.  The kernel weights $w_i$ determine the shape of the kernel, and $\lambda$ controls the concentration.  It is clear that in the case of infinite data, the kernel weights do not matter - however, this is not the case in finite data.  An appropriately shaped kernel can drastically affect the quality of probability estimates.

In the following example, we will compare conditional probability estimates using kernels with different shapes.  We will begin by drawing $n=500$ points uniformly at random from the square $[-25,25]^2$, and then a label $y \in \{0,1\}$ with the conditional probability
\begin{equation}
\label{eq:circle}
\condpr{y=1}{x} = 
\begin{cases}
1 &\mbox{if } r(x) < 8 \\
\frac{28-r(x)}{20} &\mbox{if }  8 \geq r(x) \geq 28  \\
0 & \mbox{if } r(x) \geq 1.                         
\end{cases}
\end{equation}
where $r(x) = \sqrt{\left( x_1^2 + x_2^2 \right)}$.  This is the ``circle model" from \citet{mease2007}.  Note that the level sets of the conditional class probability model are simply concentric circles, which are shown in Figure~\ref{fig:kernel_example}.  We consider probability estimation with two different kernels.  The first is $K_{\lambda, (1,1)} (x,z) = \exp{\parens{-\lambda \left( |x_1 - z_1| +  |x_2 - z_2| \right)}}$ and the second is $K_{\lambda, (10,1)} (x,z) = \exp{\parens{ -\lambda \left( 10|x_1 - z_1| +  |x_2 - z_2| \right)}}$.  Note that the first kernel is has symmetric level sets - which is well-suited for this problem - and the second has level sets which are very skewed in the $x_1$ direction.  In Figure~\ref{fig:kernel_example} we plot the RMSE for probability estimation on a holdout set using each of these kernels for various values of the bandwidth parameter $\lambda$.  The best RMSE achieved by the symmetric kernel is 0.07, while the best achieved by the skewed kernel is 0.15.  This comes as no surprise: with a finite amount of data, the shape of the kernel is critical.  Interestingly, both of these estimators achieve the same misclassification rate when used for class estimation.  We would again like to emphasize the point that misclassification is generally a much more forgiving task than probability estimation: tuning matters!

\begin{figure}[htp]
\centering
    \begin{subfigure}[b]{0.3\textwidth}
        \includegraphics[width=\textwidth]{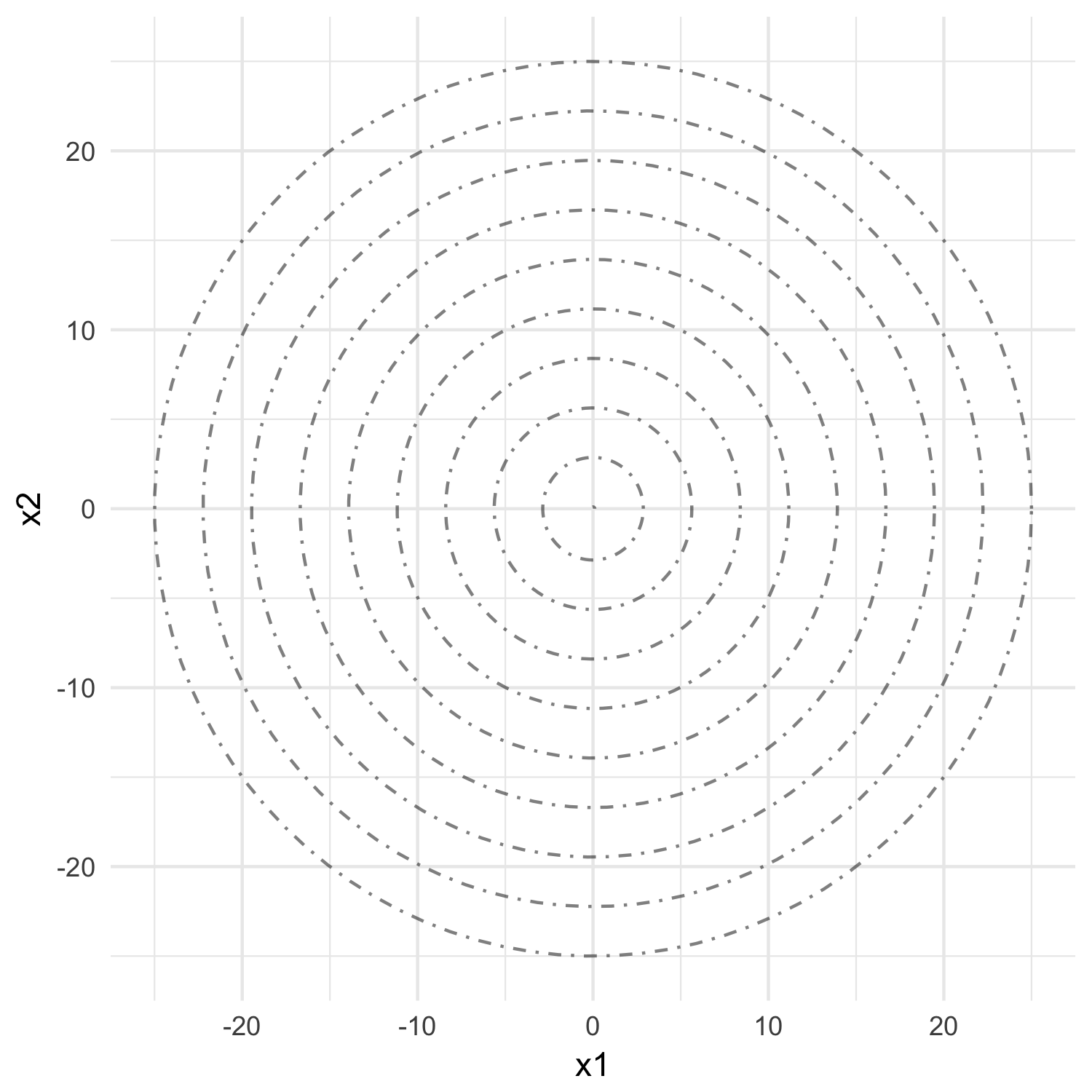}
        \caption{}
	\label{subfig:}
     \end{subfigure}
     \quad
     \begin{subfigure}[b]{0.3\textwidth}
        \includegraphics[width=\textwidth]{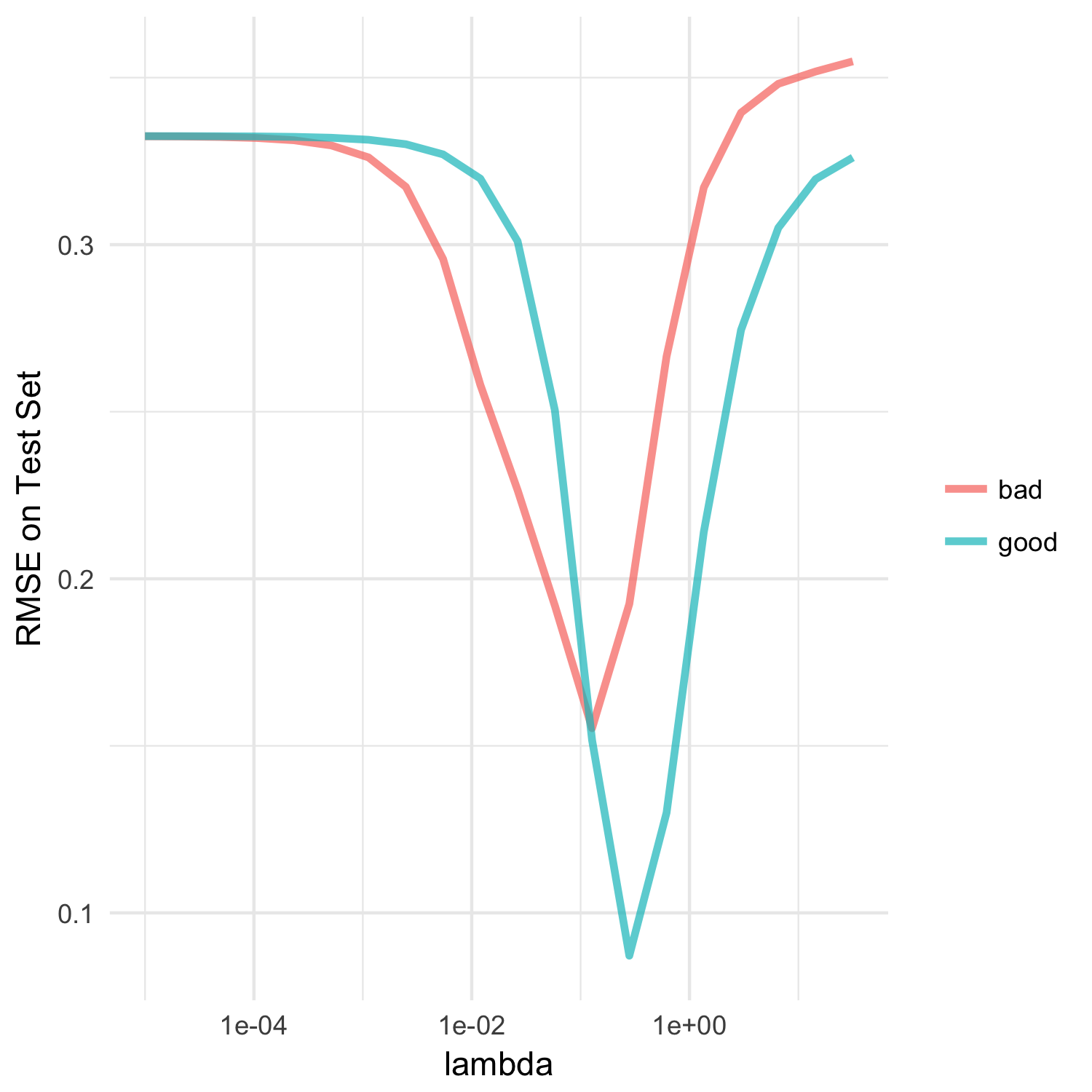}
        \caption{}
	\label{subfig:}
     \end{subfigure}
     \caption{The plot on the left shows the level sets for the conditional class probability function.  The plot on the right shows the out of sample root mean squared error for probability estimation for two different kernels.  The red line corresponds to a kernel with poor shape relative to the underlying probability density function, while the blue line shows a well-suited kernel.}
          \label{fig:kernel_example}
\end{figure}

We will close this section with a final point about kernel shape and sparsity.  Suppose we observe $n$ training points $(x_i, y_i)$, where $x \in [0,1]^p$ and $y \in \{0,1\}$ with $\condpr{y=1}{x}$ a function only of the first coordinate.  Our goal will be to estimate the probability that $y=1$ at some point $x_0$. In order to do this, we will consider counting the fraction of points in neighborhoods of $x_0$ that have volume $0 < s < 1$, but different shapes - for simplicity we will consider hyperrectangles with different side lengths.  If we consider a square, the side lengths of a neighborhood that capture a fraction $s$ of the training data will have expected side length $s^{1/p}$ in the first coordinate.  Now, consider a rectangle which has side lengths in the ratio of $\lambda : 1 : \cdots : 1$.  The expected side length in the first coordinate to capture a fraction $s$ of training points is now $\lambda \left( \frac{s}{\lambda} \right)^{1/p}$.  Thus, the ratio of requires lengths in the signal direction between the square and rectangle is $\lambda^{1/p - 1}$.  When $\lambda < 1$, that is, when the rectangle is more narrow in the signal direction, the square requires about $1/\lambda$ times more data points to estimate the probability with the same precision as the rectangle.

\section{Empirical Properties of the Proximity Function}
\label{sec:empirical_kernel}

We will now present some simulated examples to demonstrate some of the qualitative findings from the previous section.  In particular, we will show empirically that the random forest proximity matrix adapts to level sets of conditional class probability functions, and sparsity.  These findings should resonate with the intuition for probability estimation presented in Section~\ref{subsec:laplace}.

\subsection{Adaptation to CCPF Shape}
\label{subsec:adaptation_shape}

We will begin by illustrating that the proximity function's shape reflects the local geometry of the conditional class probability function.  While not implied directly by Proposition~\ref{prop:kernel}, one could easily modify the Proposition's assumptions to reflect preferential splitting among the strong variables.  This is future work, but we find it to be enlightening to present the current material to the extent that it reflects our discussion in Section~\ref{subsec:laplace}, and has direct implications for the quality of probability estimation.  Furthermore, in the next three examples we include a \textit{completely random forest} as a straw man for which to compare the actual random forest algorithm.  The \textit{completely random forest} is nonadaptive: it operates exactly as a random forest, except the predictor to split on and split value are chosen uniformly at random.  The proximity function derived from this algorithm is most similar to the one suggested in Breiman's analysis mentioned in Section~\ref{subsec:naive_model}.

Our first example consists of a piecewise logistic model, where the form of the class probability is given by Equation~\ref{eq:kinked}.  On the domain $x_1 < 0$, the probability model places more weight on the first coordinate, while on the domain $x_2 \geq 0$ the model places more weight on the second coordinate.  The structural break at $x_0$ will allow us the opportunity to investigate the extent to which the random forest locally adapts to changes in the shape of the probability function.  We draw $n=2,000$ points on the square $[-1,1]^2$, and we draw a class label $y \in \{0,1\}$ according to Equation~\ref{eq:kinked}.

\begin{equation}
\label{eq:kinked}
\condpr{y=1}{x} = 
\begin{cases}
 \frac{1}{1 + \exp \left( -3x_1 - x_2 \right)} &\mbox{if } x_1 < 0 \\
\frac{1}{1 + \exp \left( -x_1 - 3x_2 \right)} &\mbox{if } x_1 \geq 0
\end{cases}
\end{equation}

In Figure~\ref{fig:kinked} we plot level sets for the random forest and completely random forest proximity functions at different points.  In particular, the top set of figures show level sets for the random forest, and the bottom set of figures show level sets for the completely random forest.  The left set of figures shows level sets centered at $x_0 = (-0.25, -0.25)$, and the right set of figures shows level sets centered at $x_0 = (0.25, 0.25)$.  Let us first turn our attention to Figure~\ref{subfig:rf_ll}.  Since the first coordinate of the target point $x_0 = (-0.25, -0.25)$ is negative, the probability model places more weight on the first coordinate.  Correspondingly, we see that the level set for the random forest kernel is most narrow in the direction of the first coordinate.  Symmetrically, we see that the kernel is most narrow in the second coordinate in Figure~\ref{subfig:rf_ur}, in which the kernel is centered at $x_0 = (0.25, 0.25)$.  When parsing these findings, it is good to keep in mind the circle example from Section~\ref{subsec:laplace} in which we demonstrated the importance of the shape of the kernel.

One should compare the shape of the random forest kernel in each case with that of the \textit{completely random forest}, which is plotted in the bottom set of figures in Figure~\ref{fig:kinked}.  At both test points, the level sets are symmetric, regardless of the shape of the class probability function.  Of course, this is not surprising: the \textit{completely random forest} is nonadaptive by design.  Finally, averaged over 50 replications, the average difference in root mean square error between the \textit{completely random forest} and the random forest (using the kernel approach) at the point $x_0 = (-0.25, -0.25)$ is 0.002, while at $x_0 = (0.25, 0.25)$ is 0.007.

\begin{figure}[htp]
\centering
    \begin{subfigure}[b]{0.3\textwidth}
        \includegraphics[width=\textwidth]{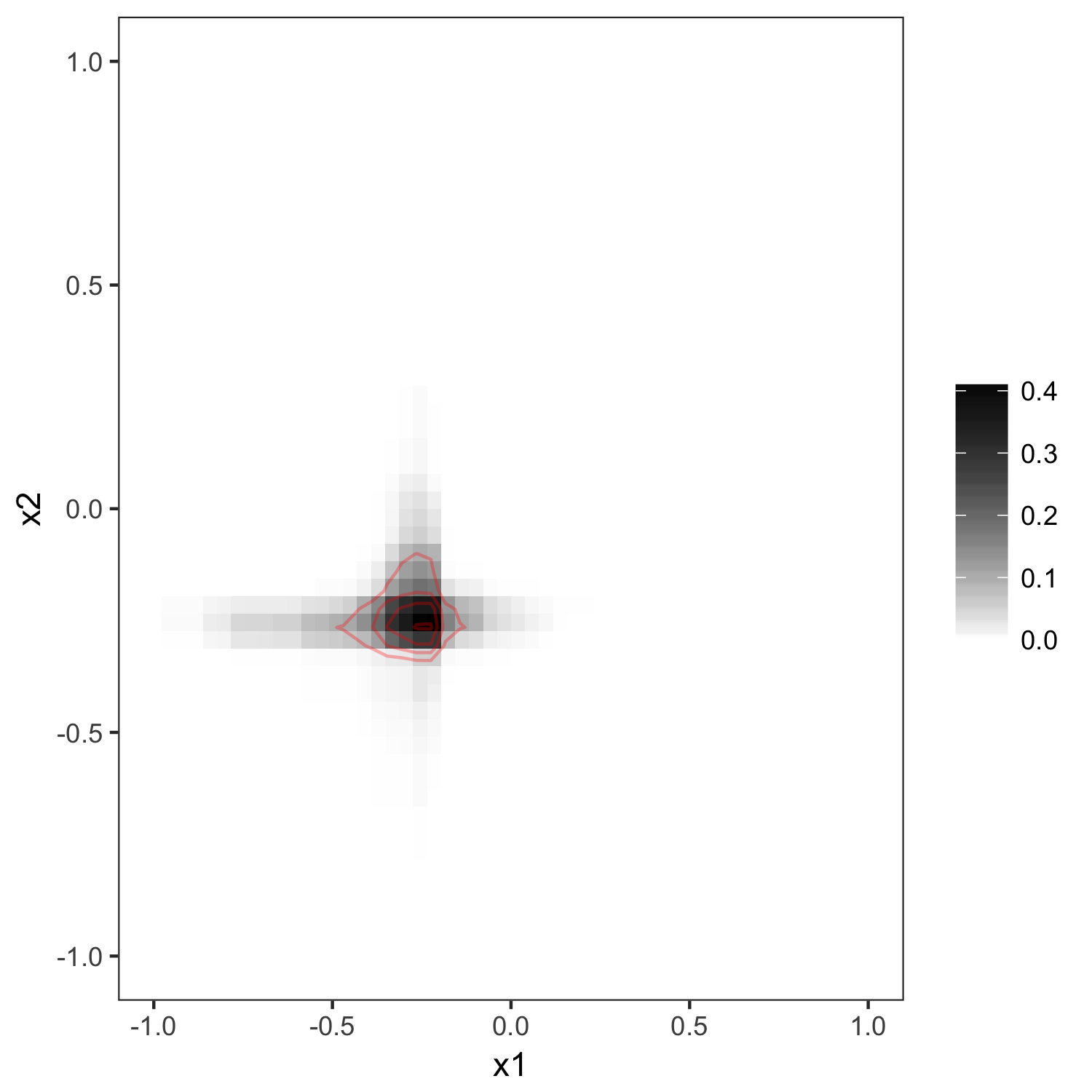}
        \caption{}
  \label{subfig:rf_ll}
     \end{subfigure}
     \quad
          \begin{subfigure}[b]{0.3\textwidth}
        \includegraphics[width=\textwidth]{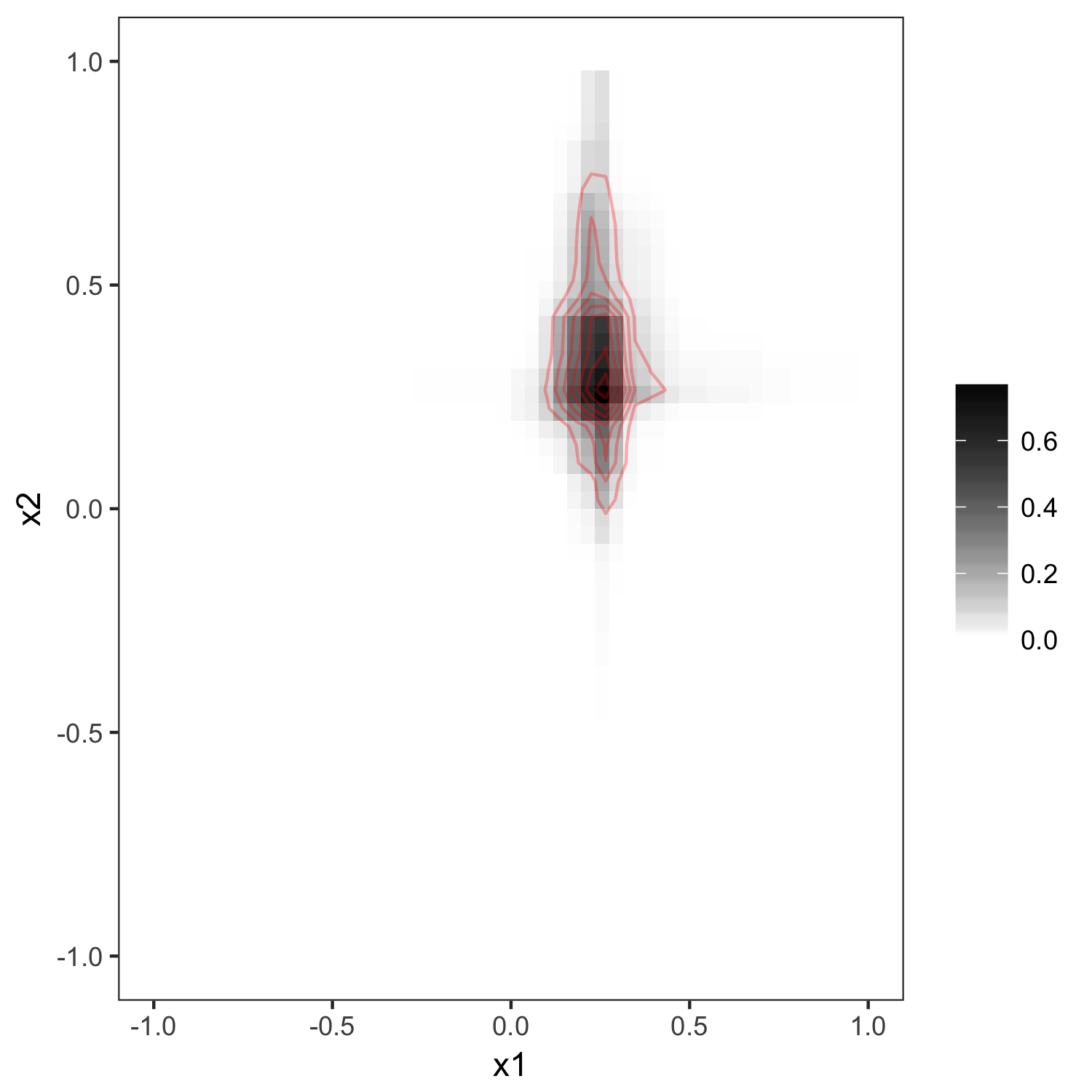}
        \caption{}
  \label{subfig:rf_ur}
     \end{subfigure} \\
         \begin{subfigure}[b]{0.3\textwidth}
        \includegraphics[width=\textwidth]{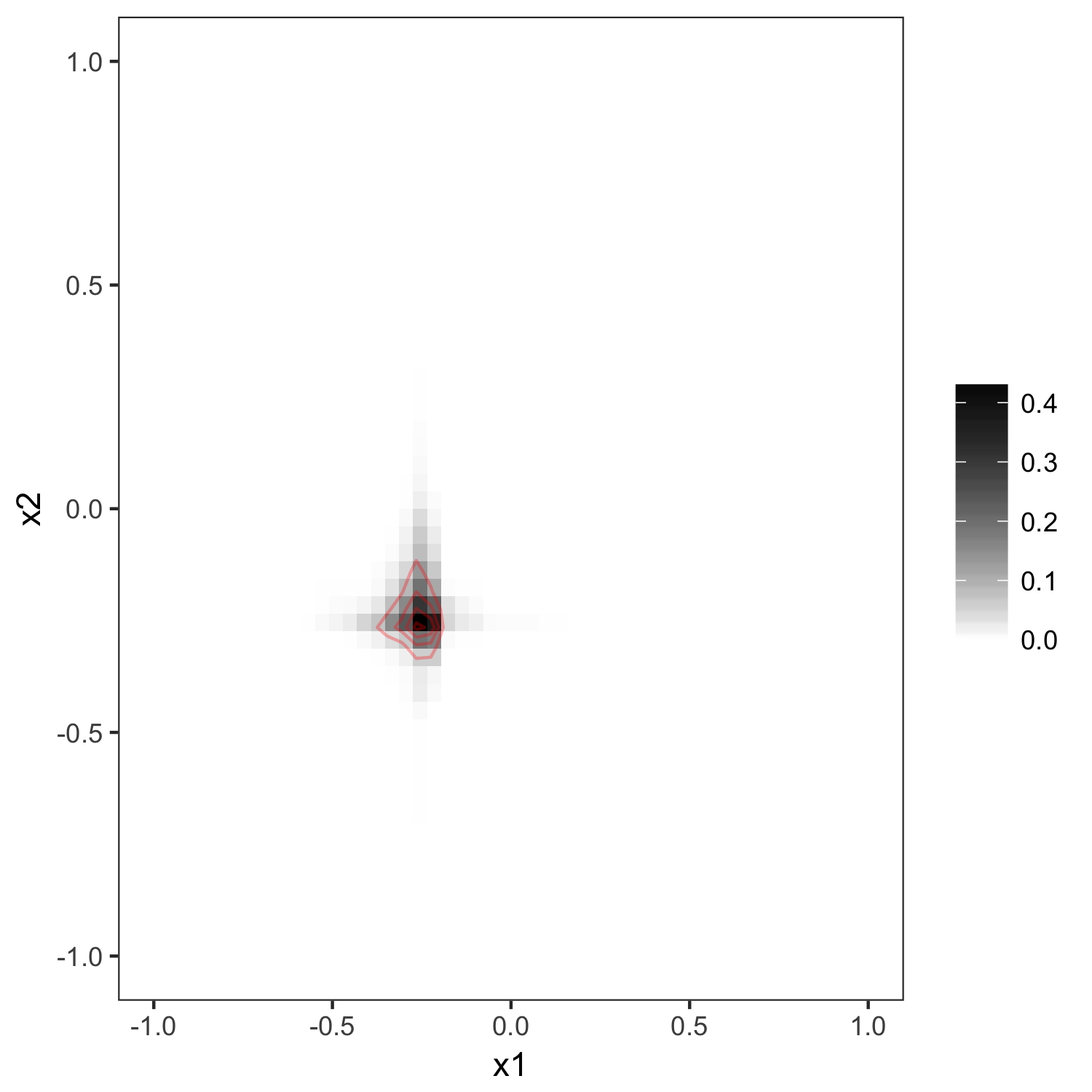}
        \caption{}
  \label{subfig:crf_ll}
     \end{subfigure}
     \quad
     \begin{subfigure}[b]{0.3\textwidth}
        \includegraphics[width=\textwidth]{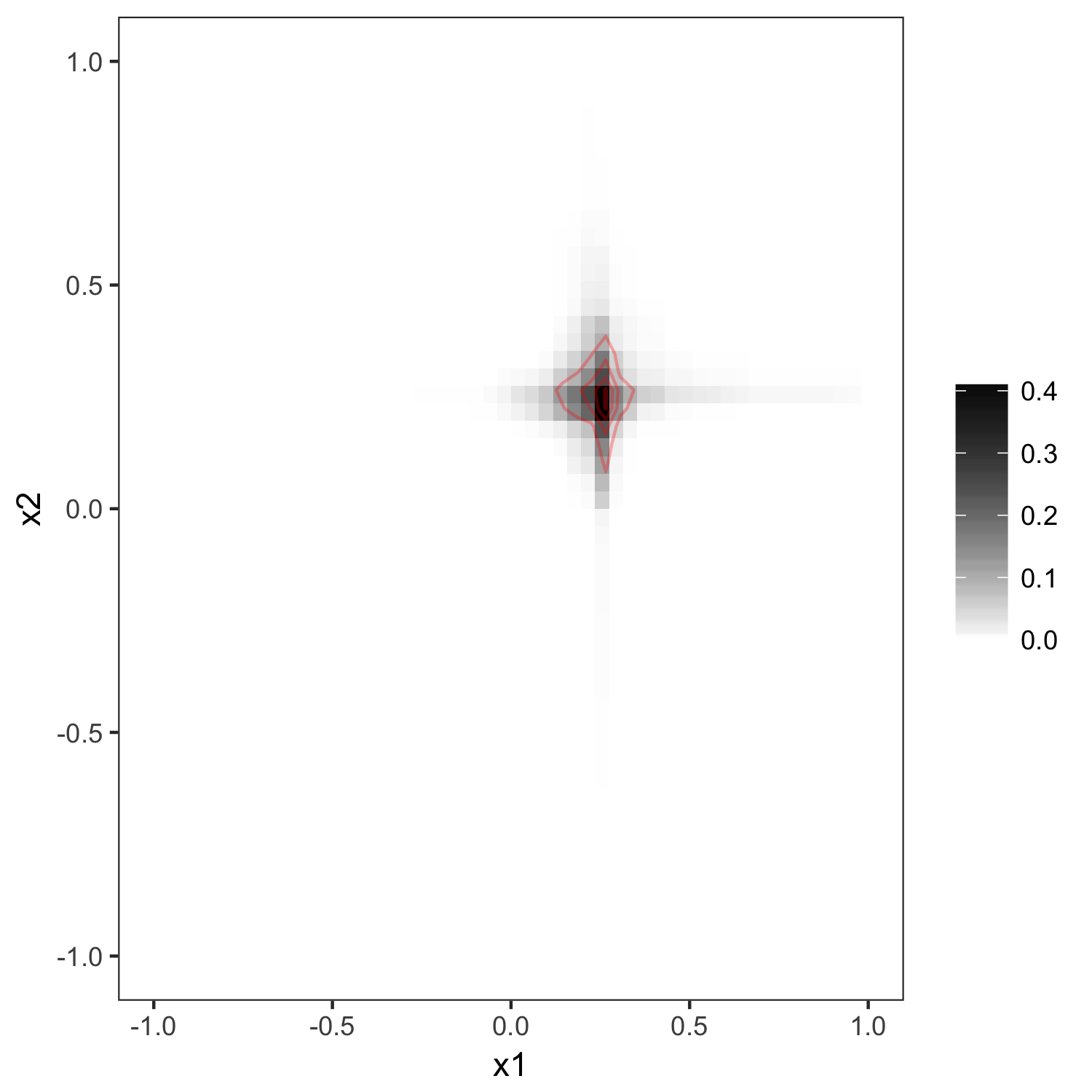}
        \caption{}
  \label{subfig:crd_ur}
     \end{subfigure}\\
     \caption{Level sets for the estimated probability density at a point $x_0$ for both random forests (sub-figures (a) and (b)) and completely random forests (sub-figures (c) and (d)).  Note that in the left set of figures, the target point is $x_0 = (-0.25, -0.25)$, while in the right set of figures $x_0 = (0.25, 0.25)$.  It is clear that that the level sets produced by the completely random forests are much more symmetric than the usual random forest.}
          \label{fig:kinked}
\end{figure}

\subsection{Adaptation to Sparsity}
\label{subsec:adaptation_sparsity}

In our naive kernel model from Section~\ref{subsec:naive_model} we argued that a random forest attached greater weights to signal variables than noise variables, and it is through this mechanism that it adapts to sparsity.  A random forest's ability to adapt to sparsity has been argued in other places in the literature, as well as from a good track record of prediction in high dimensional problems \citep{biau2012}, \citep{scornet2015}, \citep{diaz2006}.  We would like to illustrate through two examples that this adaptivity is reflected in the shape of the proximity function, which is narrow in directions of signal and flat in directions of noise.

Our first example is a simple logistic regression model, which we choose for visual ease.  We draw $n=500$ predictors $\x$ uniformly at random from the square $[-1,1]^2$, and then we draw a label $y \in \{0,1\}$ according to the probability
\begin{equation*}
\condpr{y=1}{x} = \frac{1}{1 + \exp\{ -3x_1\}}.
\end{equation*}
Note that only the first coordinate of the predictor $x$ matters: the second coordinate is a noise variable.  In this setting, when considering neighborhoods of a point $x_0$ to use for probability estimation, the best neighborhoods consist rectangles that are much more thin in the $x_1$ direction than the $x_2$ direction.  In Figure~\ref{fig:logistic2d}, we plot the level curves for the proximity function centered at $x_0 = (0,0)$ using a random forest and a \textit{completely random forest} in the left and right plots, respectively.  The random forest proximity function has the shape we would expect at the origin: a thin vertical strip.  On the other hand, the \textit{completely random forest} has a neighborhood that is symmetric in both the $x_1$ and $x_2$ directions.  In higher dimensions, one pays a price for this symmetry, as will be explored in more extensive simulation results in Section~\ref{sec:prob_comp}.

\begin{figure}[htp]
\centering
    \begin{subfigure}[b]{0.3\textwidth}
        \includegraphics[width=\textwidth]{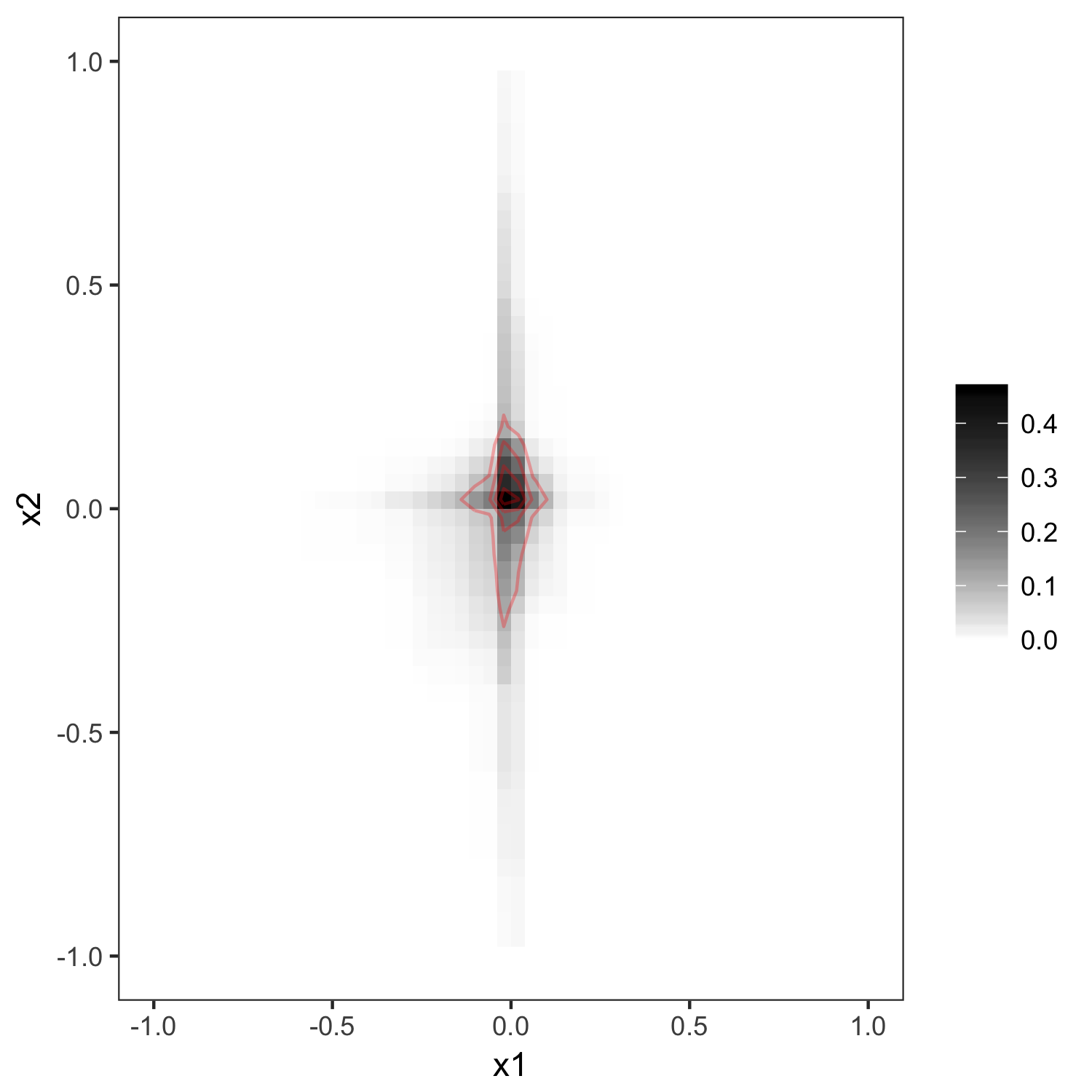}
        \caption{Random forest}
  \label{subfig:}
     \end{subfigure}
     \quad
     \begin{subfigure}[b]{0.3\textwidth}
        \includegraphics[width=\textwidth]{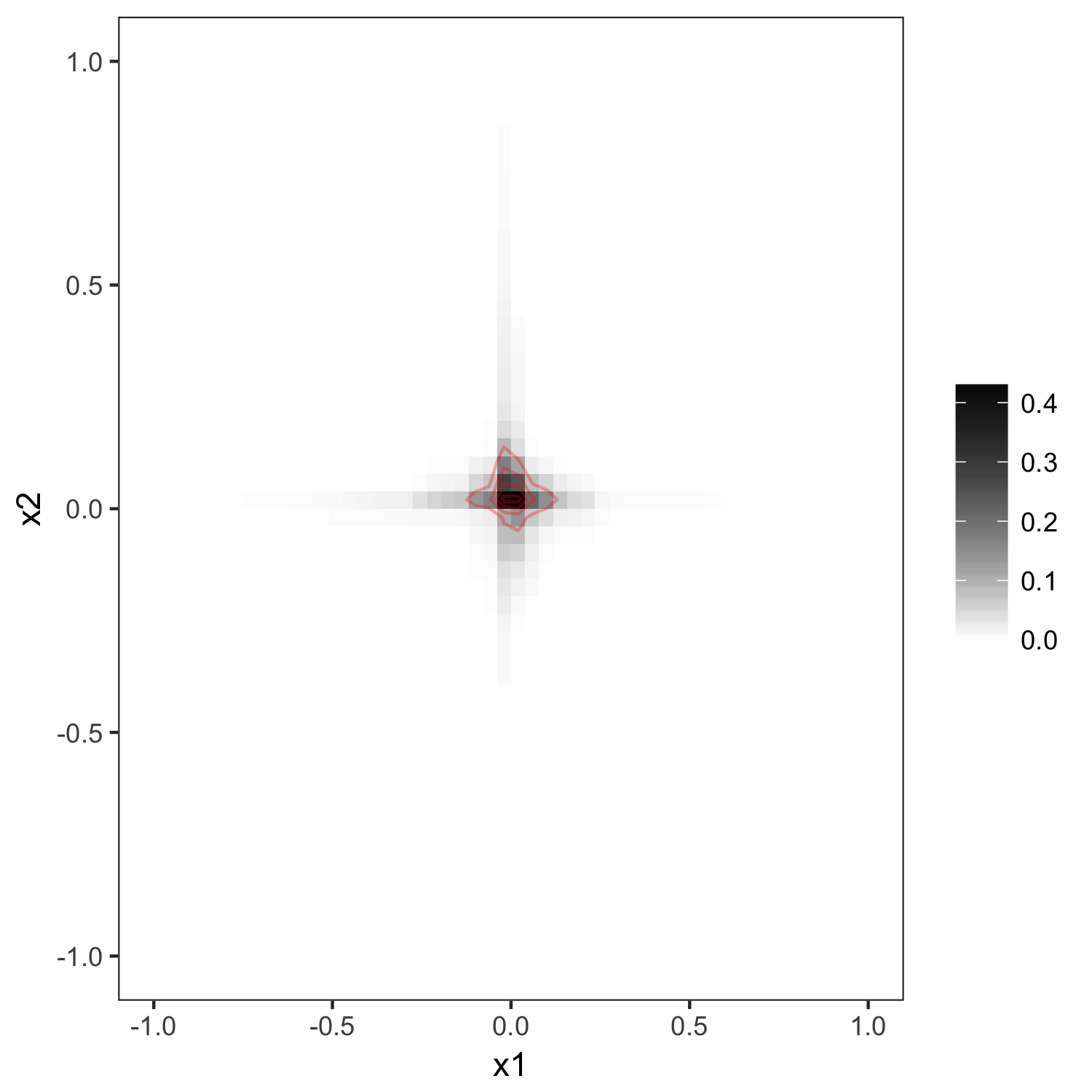}
        \caption{Completely random forest}
  \label{subfig:}
     \end{subfigure}
     \caption{Level sets for the estimated probability density at the point $x_0 = (0,0)$. The level set for the random forest stretches out in the signal dimension $x_1$, while the level set for the completely random forest is symmetric.}
          \label{fig:logistic2d}
\end{figure}

In our next example, we continue using a logistic model, but with a greater number of noise variables.  We draw $n=1,000$ predictors $\x$ uniformly at random from $[-1,1]^{22}$, and conditional on $\x$, we draw $y$ with probability 
\begin{equation*}
\condpr{y=1}{x} = \frac{1}{1 + \exp\{ -2x_1 - 2x_2\}}.
\end{equation*}
Notice that only the first two dimensions contain signal about the class label $y$.  The remaining 20 coordinates consist of noise variables.  Unlike the previous example, the level sets of the proximity function are difficult to visualize, so we instead investigate approximate directional derivatives of the proximity function centered at the origin
\begin{align*}
D_{j+} K(0,0) &\approx \frac{K(h e_j, 0) - K(0,0)}{h} \\
D_{j-} K(0,0) &\approx \frac{K(-h e_j, 0) - K(0,0)}{h} \\
\end{align*}
where we take $h=0.25$.  Figure~\ref{fig:derivative} plots the directional derivatives for each of the 22 directions in predictor domain.  The points in red are the values of the derivative for the random forest, and those in blue are the values for the \textit{completely random forest}.  Turning our attention first to the red points in Figure~\ref{subfig:derivative_pos}, the derivatives in the the $x_1$ and $x_2$ directions for the random forest kernel are $-1.9$ and $-2.85$, respectively, while the derivatives in the other dimensions are close to zero.  In other words, the kernel is very peaked in the signal dimensions, and very flat in the noise dimensions.  On the other hand, the directional derivatives for the \textit{completely random forest} kernel are all approximately the same, indicating a symmetric, flat shape.  The same phenomenon holds when analyzing left hand derivatives in Figure~\ref{subfig:derivative_neg}.  It further holds that as we increase the value of $\mtry$, the kernel becomes increasingly peaked in the directions of signal as predicted by our naive model.  Analyzing such directional derivatives might be an interesting alternative to variable importance type measures in a random forest.

\begin{figure}[htp]
\centering
    \begin{subfigure}[b]{0.3\textwidth}
        \includegraphics[width=\textwidth]{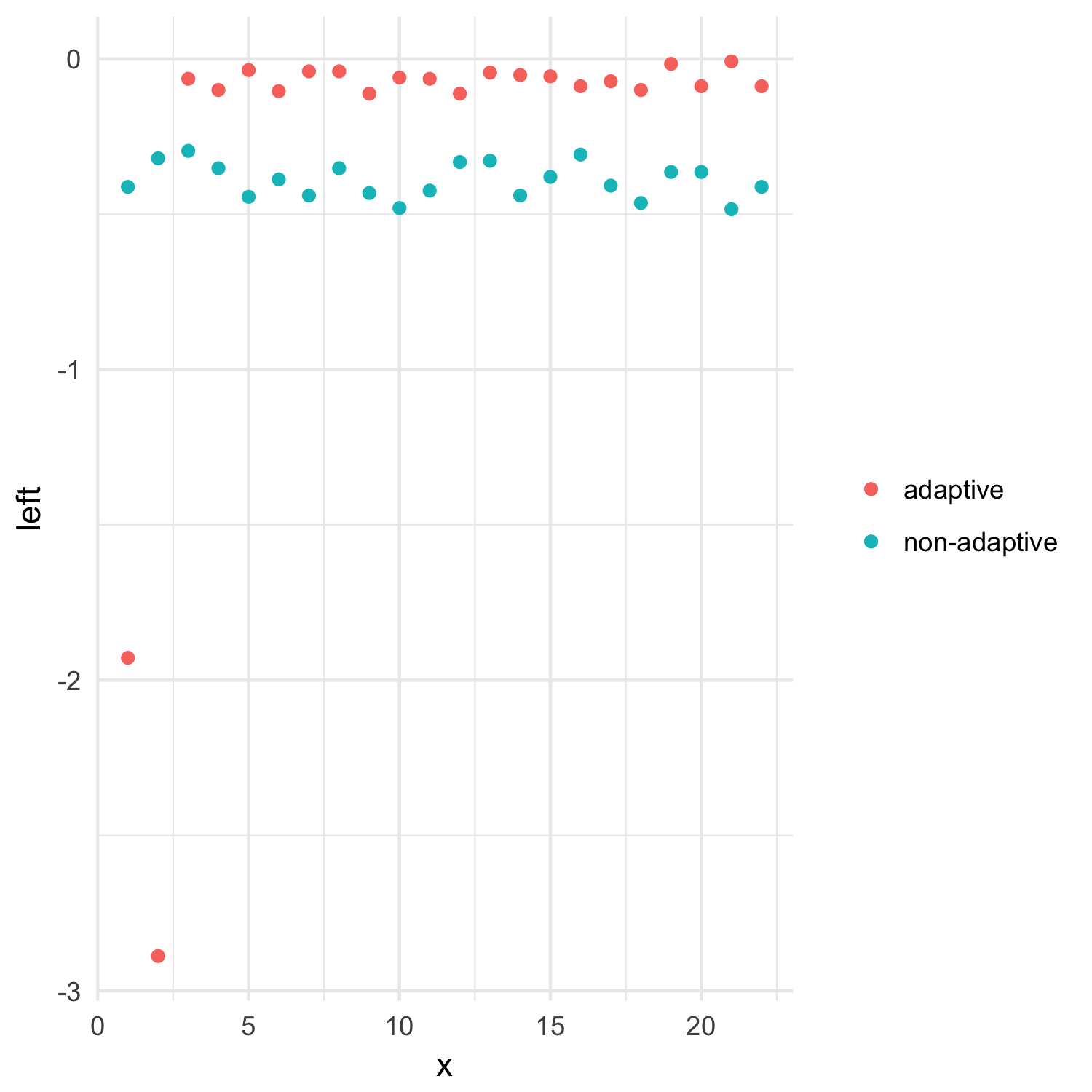}
        \caption{Left derivative}
  \label{subfig:derivative_pos}
     \end{subfigure}
     \quad
     \begin{subfigure}[b]{0.3\textwidth}
        \includegraphics[width=\textwidth]{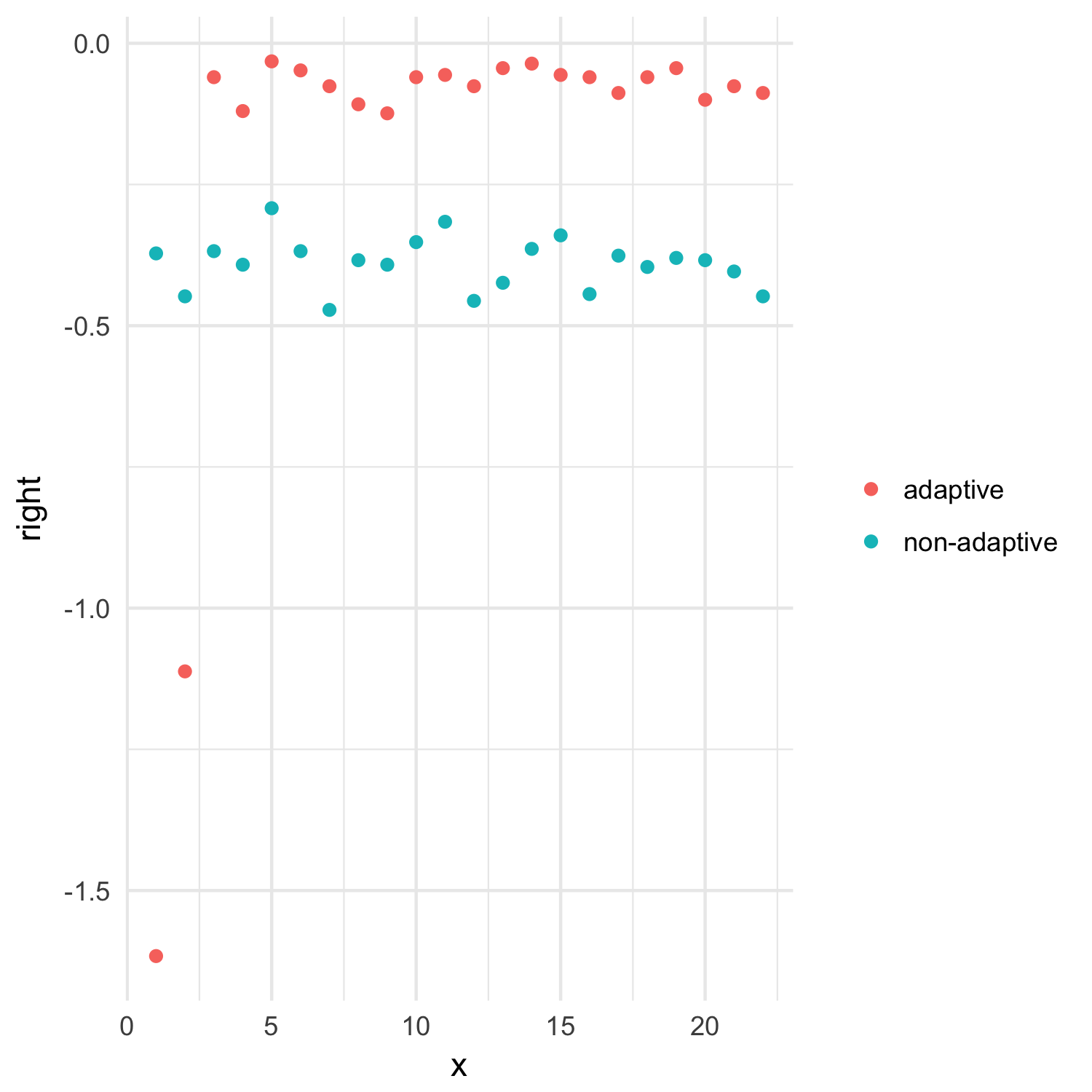}
        \caption{Right derivative}
  \label{subfig:derivative_neg}
     \end{subfigure}
     \caption{Plots of the left and right directional derivatives of the estimated probability density function produced by random forests (adaptive) and completely random forests (non-adaptive).}
          \label{fig:derivative}
\end{figure}

\subsection{Kernel Shape: the Spam Data Example}
\label{subsec:spam}

In Section~\ref{sec:kernel_intuition} we proposed a simple model of a random forest kernel that connected the shape of the kernel to the $\mtry$ parameter.  Specifically, we argued that as $\mtry$ increased, the kernel became more concentrated on signal variables, and more \textit{flat} in noise dimensions.  This intuition was also reflected on our motivating example in Section~\ref{subsec:motivation}.  Here, we will provide an illustration of this point on the \textit{spam} dataset.  

The \textit{spam} data set consists of $n=4601$ emails, along with $p=57$ predictors, with each predictor giving the frequency of certain words in that email.  Each example is attached with a label indicating whether than email is spam or not.  We are interested in determining the shape of the random forest kernel in the direction of noise variables.  Since we do not know these variables ahead of time in this data set, we add in 50 additional ``junk" predictors in the data that consist of randomly permuted predictors.  By permuting the predictors, we break any association with our newly constructed predictor and the response.

We fit a classification random forest for different values of $\mtry$, and we compute directional derivatives of the kernel as in Section~\ref{subsec:adaptation_sparsity}.  More precisely, we estimate directional derivatives at the point $x_0$, where the $p=107$ components of $x_0$ are the sample means for each predictor (including the ``junk" predictors).  The directional derivative estimates are given by
\begin{equation*}
D_{j+} K(x_0,x_0) \approx \frac{K(x_0 + h_j e_j, x_0) - K(x_0,x_0)}{h_j} \\
\end{equation*}
where $h_j$ is equal to $0.2$ times the standard deviation of the $j^{th}$ predictor.

Figure~\ref{fig:spam_deriv} plot the average absolute value of the directional derivative in the 50 ``junk" directions as a function of $\mtry$.  As predicted, the typical size of this derivative decreases as the value of $\mtry$ increases.  This is not surprising: larger values of $\mtry$ make it more likely that signal variables are used in tree splits, decreasing the influence of noise variables.  One can visualize this effect as a ``flattening" of the kernel in noise dimensions as $\mtry$ increases.
\begin{figure}[htp]
\centering
        \includegraphics[width=0.3\textwidth]{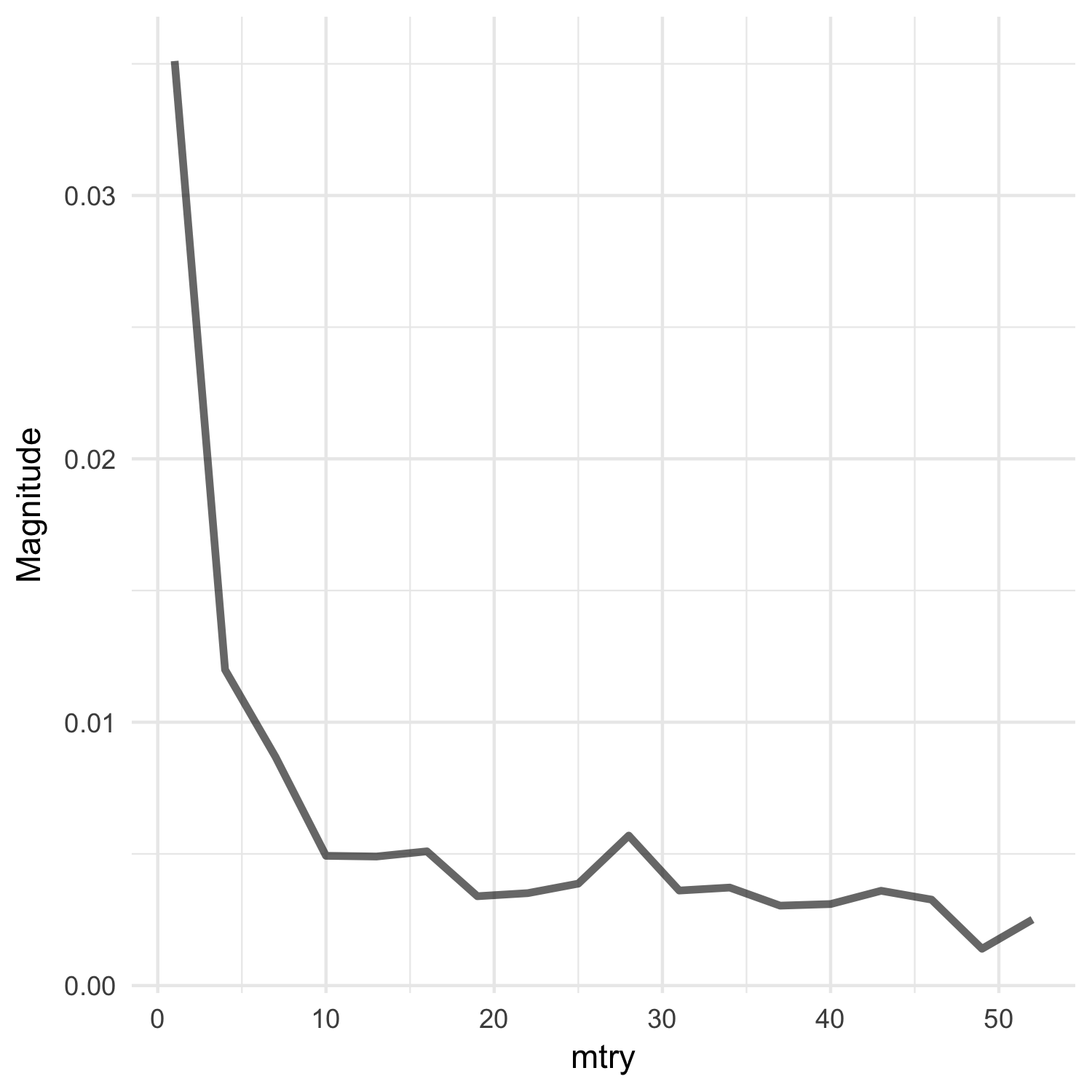}
        \caption{The average absolute value of the magnitude of directional derivative in the artificial noise variables in the \textit{spam} data set.}
  \label{fig:spam_deriv}
\end{figure}
We provide another illustration of the kernel shape in Figure~\ref{fig:spam}.  In each figure, the vertical axis relates the value of the directional derive for each coordinate $j = 1, \ldots, 107$, given by the horizontal axis.  Figure~\ref{subfig:spam_mtry10} shows this relationship for a random forest fit with $\mtry$ set to the classification forest default of $\sqrt{p} = 10$, while Figure~\ref{subfig:spam_mtry34} shows this relationship for the regression forest default of $p/3 = 34$.  In each figure, it is immediate to notice that the last 50 predictors have very small directional derivatives, given by black dots close to zero.  When $\mtry$ increases from 10 to 34, the values of all directional derivatives get shrunken to zero, with the exception of a few large values for predictors in the original data set.

\begin{figure}[htp]
\centering
    \begin{subfigure}[b]{0.3\textwidth}
        \includegraphics[width=\textwidth]{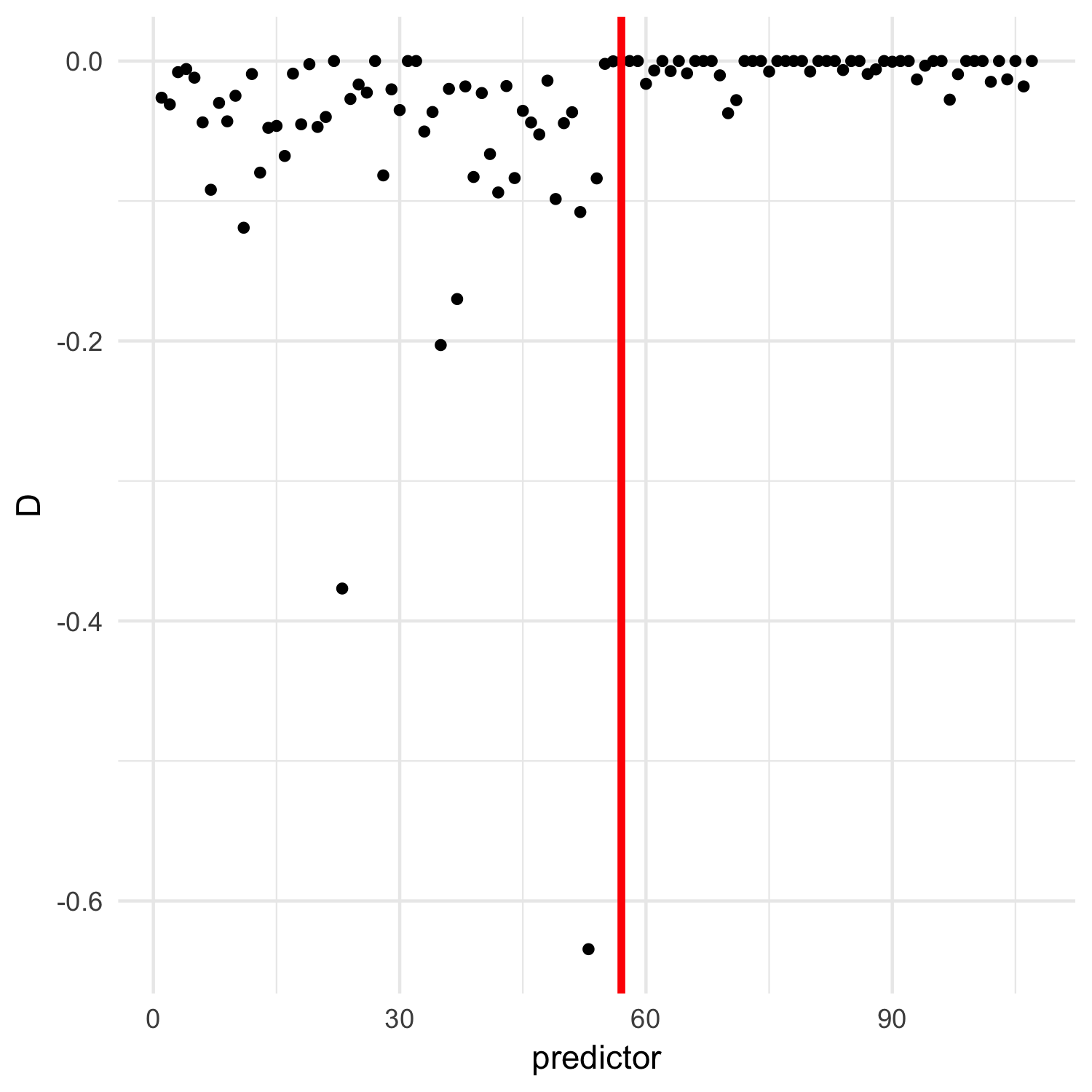}
        \caption{$\mtry = 10$}
  \label{subfig:spam_mtry10}
     \end{subfigure}
     \quad
     \begin{subfigure}[b]{0.3\textwidth}
        \includegraphics[width=\textwidth]{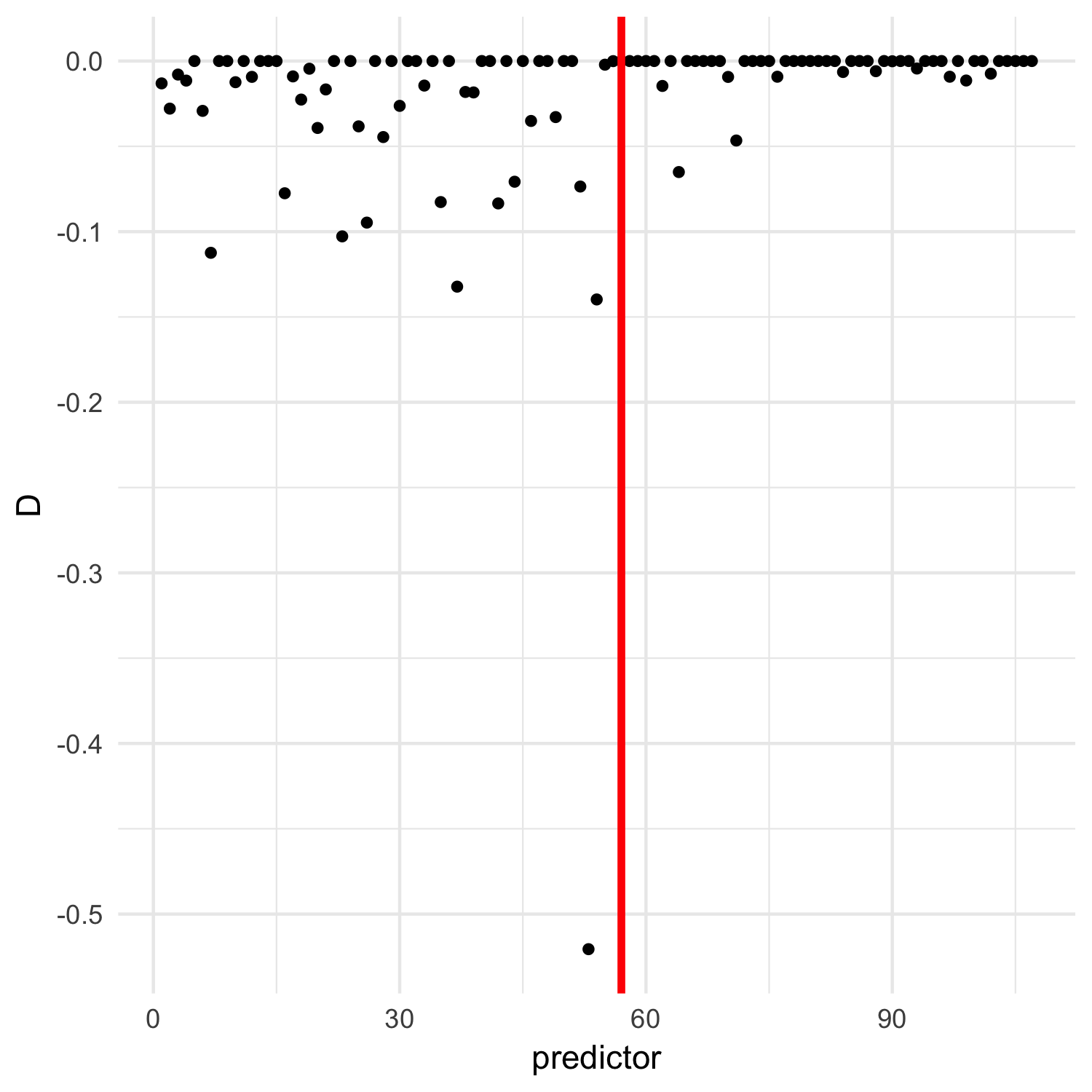}
        \caption{$\mtry = 34$}
  \label{subfig:spam_mtry34}
     \end{subfigure}
     \caption{Directional derivatives for each variable in the \textit{spam} data set.  Variables with indices to the right of the red line are artificially contructed noise variables.}
          \label{fig:spam}
\end{figure}

\section{Probability Comparisons}
\label{sec:prob_comp}

We will now undertake an empirical investigation to determine the extent to which random forest parameter tuning influences its probability estimates in a number of simulated and real data examples.  Specifically, we will be interested in studying the affect of $\mtry$ on probability estimates, as well as the type of forest used, i.e. regression, classification, or proximity.  In particular, we will consider regression and classification random forests with values at $\mtry$ set at both $p/3$ and $\sqrt{p}$, proximity random forests tuned for the best value of $\mtry$, completely random forests, and bagged CART trees.

\subsection{Real Data Sets}
\label{subsec:real_data}

We begin by considering probability estimation in six data sets taken from the UCI machine learning repository: \textit{spam}, \textit{splice}, \textit{tic}, \textit{parkinsons}, \textit{australian credit}, and \textit{ionosphere}.  A description of the original data is contained in Appendix~\ref{sec:data_description}.  Unlike the simulated examples presented earlier in the paper, we do not have the actual probabilities in our data sets, so we will evaluate our probability estimates according to an empirical root mean squared error:
\begin{equation*}
\sqrt{\sum^n_{i=1} \parens{\widehat{p}_i - y_i}^2}.
\end{equation*}
This quantity is computed for 50 random 80/20 splits of the data, and we report the average.  Furthermore, we add in 50 ``junk" predictors to each data-set in the manner described in Section~\ref{subsec:spam}.  We added in noisy predictors to investigate the efficacy of random forests in relatively sparse settings, as well as to consider a wider variety of $\mtry$ settings in data sets with small numbers of predictors.

\begin{table}[ht]
\centering
\begin{tabular}{cccccccc}
  \hline
 &  $RF^{class}$ & $RF^{reg}$ & $RF^{class}$ & $RF^{reg}$ & Bagged & Random & $RF^{prox}$ \\ 
 $\mtry$ & $p/3$ & $p/3$ &  $\sqrt{p}$ &  $\sqrt{p}$ & $p$ & One & Best \\
 \hline
\textit{spam} & 0.200 & 0.201 & 0.196 & 0.198 & 0.208 & 0.211 & 0.207 \\ 
  \textit{splice} & 0.119 & 0.120 & 0.159 & 0.161 & 0.122 & 0.241 & 0.124 \\ 
  \textit{tic} & 0.163 & 0.179 & 0.183 & 0.200 & 0.135 & 0.287 & 0.184 \\ 
  \textit{parkinsons} & 0.257 & 0.260 & 0.257 & 0.259 & 0.268 & 0.280 & 0.313 \\ 
  \textit{australian credit} & 0.310 & 0.308 & 0.312 & 0.312 & 0.315 & 0.328 & 0.308 \\ 
  \textit{ionosphere} & 0.227 & 0.230 & 0.225 & 0.227 & 0.245 & 0.287 & 0.262 \\ 
   \hline
\end{tabular}
\caption{Table of root mean squared errors for probability estimates for each of 6 data sets from the UCI repository.  Values of RMSE were computed as an average over 50 random training and testing splits for each dataset.  The probability estimation methods considered are random forests in regression and classification modes with different settings of $mtry$, bagged trees, completely random forests, and proximity random forests.} 
\label{table:real_data}
\end{table}

The first set of results are reported in Table~\ref{table:real_data}.  Here, we record the root mean square error for the seven methods previously described, evaluated on six data sets.  At a high level, no method uniformly dominates any other.  However, the completely random forest performs the worst in each experiment.  This is unsurprising, especially in light of our discussion of the shape of the completely random forest kernel described earlier in the paper.  Each predictor and each split point are chosen uniformly at random, so the forest does not adapt to the shape of the data.  Bagged trees perform the best on one data set, and $RF_{prox}$ performs the best one one data set.  The reader may recall that our random forest model predicted that the bagged trees would have a kernel tht was the most narrow in the signal dimensions since it exhaustively searched over predictors at each split.  This is not neccesarily at odds with our results: bagged trees are also much more shallow as a result of the better splits, which results in a smaller effective bandwidth.

One of the aims of this paper was to argue that the tuning parameter $\mtry$ matters much more for succesful probability estimation than the type of the forest.  There have been claims in the literature that regression random forests are preferred for probability estimation because of their interpretation as `conditional expectation machines' \citep{malley2012}, \citep{kruppa2014}.  However, these experiments ignore the fact that regression and classification forests have different default settings of $\mtry$ in common software, and we argue this is a crucial confounding factor.  For each experiment, we conducted a paired t-test comparing the root mean square error for regression and classification forests for the two different levels of $\mtry$ used as defaults in existing software: $p/3$ and $\sqrt{3}$.  In all but the \textit{tic} dataset, we found the root mean square errors produced by classification and regression forests to be indistinguishable.

We also considered a more in-depth analysis of the \textit{splice} data set.  Figure~\ref{fig:splice} displays the misclassification error (test error) and root mean square error as a function of $\mtry$.  It is clear from Figure~\ref{subfig:splice_test} that test error is relatively immune to the value of $\mtry$: for values larger than 4, test error is roughly the same.  The story is quite different for RMSE, as shown in Figure~\ref{subfig:splice_rmse}.  Here, the fall in RMSE is much more gradual as a function of $\mtry$, plateauing near a value of 30.  It is of interest to note that if one were to optimize the forest for test error and used $\mtry = 4$, the corresponding value of RMSE would be around 0.22, which is far away from the best value of 0.1.  This example reiterates our point that tuning random forest parameters can matter substantially depending upon the quantities one wishes to estimate.  In general, test error is much less sensitive to tuning parameters than probability estimation error.

\begin{figure}[htp]
\centering
    \begin{subfigure}[b]{0.3\textwidth}
        \includegraphics[width=\textwidth]{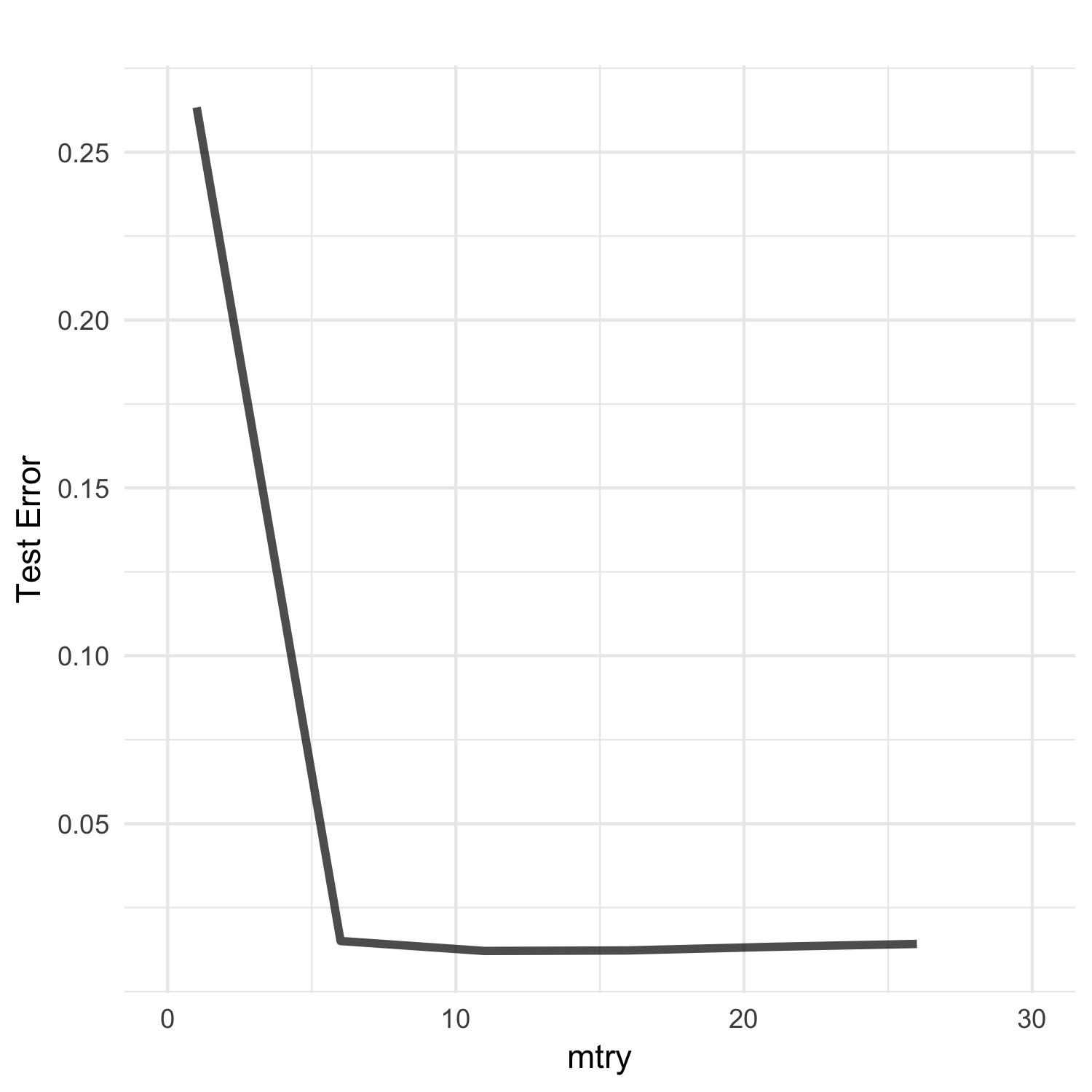}
        \caption{}
  \label{subfig:splice_test}
     \end{subfigure}
     \quad
     \begin{subfigure}[b]{0.3\textwidth}
        \includegraphics[width=\textwidth]{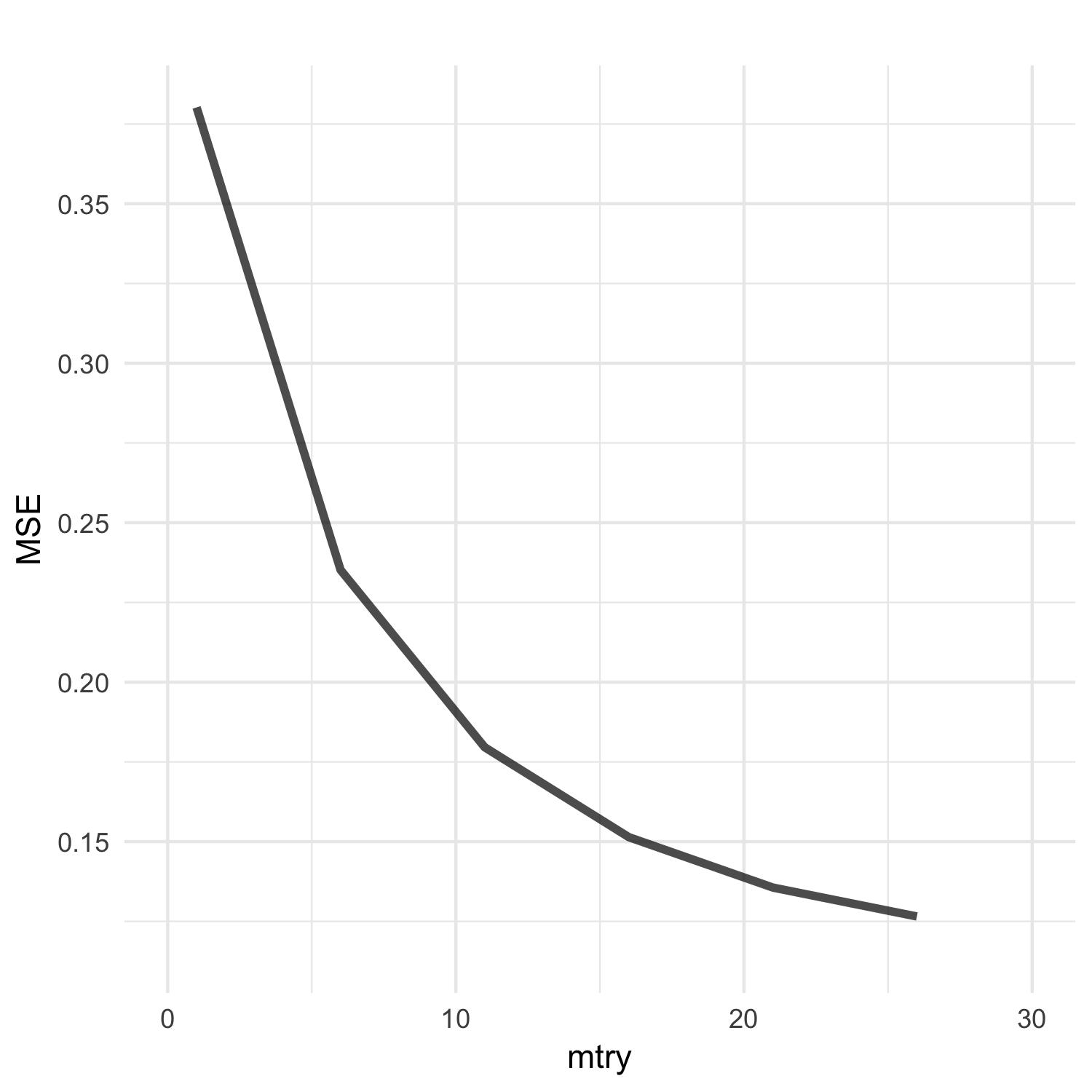}
        \caption{}
  \label{subfig:splice_rmse}
     \end{subfigure}
     \caption{Test error and root mean square probability estimation error for the \textit{splice} data as a function of $\mtry$.  Test error is generally much less sensitive to this parameter value.}
          \label{fig:splice}
\end{figure}

\subsection{Synthetic Data}
\label{subsec:synthetic_data}

In the previous experiments, we did not know the actual probability in each case, so we needed to use a surrogate measure for the quality of our probability estimates.  Here, we will analyze six simulated data sets in order to compare the different methods of obtaining probability estimates.  In each instance we will draw a training set of size $n=500$, fit seven different models, and evaluate each model on a test set of size $n=1000$.  This process is repeated 50 times, and we consider the average root mean square error over these repetitions.

The first three models in our experiments have appeared elsewhere in this paper: the \textit{circle model} was mentioned in Section~\ref{subsec:laplace} and a simple model with signal in only one dimension was used as a motivating example in Section~\ref{subsec:motivation}.  We also consider a simple logistic regression model in three dimensional space, as well as a more complicated logistic model that contains multiplicative interactions, the \textit{10d Model}.  Our last model generated data according to the \textit{XOR} function in two dimensions.  Please see Appendix~\ref{sec:data_description} for mathematical details.  This set of models captures a range of different statistical phenomena, ranging from simple probability density surfaces in low dimensions to more complicated nonlinear response surfaces.

The results are reported in Table~\ref{table:synthetic_data}.  The $RF^{prox}$ method has much better performance on the synthetic data sets, achieving the lowest average root mean squared error on four of six examples.  In some settings, such as the \textit{XOR} or \textit{1d Model}, it suffers only half as much RMSE as its nearest competitor.  Bagged trees and the completely random forest fail to achieve the best performance on any of the six data sets considered here.  The difference in error between classification and regression random forests for fixed values of $\mtry$ are larger than in the previous section, but one method does not dominate the other.  The best type of forest and tuning parameter $\mtry$ is heavily problem dependent.

\begin{table}[ht]
\centering
\begin{tabular}{cccccccc}
  \hline
 &  $RF^{class}$ & $RF^{reg}$ & $RF^{class}$ & $RF^{reg}$ & Bagged & Random & $RF^{prox}$ \\ 
 $\mtry$ & $p/3$ & $p/3$ &  $\sqrt{p}$ &  $\sqrt{p}$ & $p$ & One & Best \\
 \hline
\textit{Circle Model} & 0.161 & 0.149 & 0.179 & 0.168 & 0.179 & 0.150 & 0.137 \\ 
  \textit{1d Model} & 0.207 & 0.189 & 0.223 & 0.208 & 0.223 & 0.181 & 0.097 \\ 
  \textit{1d Model Sparse} & 0.121 & 0.120 & 0.138 & 0.137 & 0.132 & 0.181 & 0.126 \\ 
  \textit{10d Model} & 0.370 & 0.371 & 0.370 & 0.371 & 0.369 & 0.383 & 0.396 \\ 
  \textit{Logistic Model} & 0.162 & 0.151 & 0.177 & 0.168 & 0.184 & 0.136 & 0.115 \\ 
  \textit{XOR} & 0.215 & 0.197 & 0.229 & 0.213 & 0.229 & 0.193 & 0.136 \\ 
   \hline
\end{tabular}
\caption{Table of root mean squared errors for probability estimates for each of 6 simulated models .  Values of RMSE were computed as an average over 50 random training and testing splits for each dataset.  The probability estimation methods considered are random forests in regression and classification modes with different settings of $mtry$, bagged trees, completely random forests, and proximity random forests.} 
\label{table:synthetic_data}
\end{table}

\section{Conclusion}
\label{sec:conclusion}

In the statistics literature, the way in which one accomplishes probability density estimation in a nonparametric setting is through the use of (Parzen) kernels.  The canonical reference for this methodology is the \cite{parzen1962}, `On Estimation of a Probability Density Function and Mode.'  Random forests have also been found to estimate probabilities well in some settings, but they do so in an ostensibly different way, by averaging the votes of trees.  The connection that this paper makes is to frame random forest probability estimation in a familiar statistical setting.

As discussed in the Parzen paper, the kernel and its associated bandwidth parameter are crucial quantities in determine the quality of density estimate.  In a random forest, it is not obvious what these two quantities are, or how one might tune them.  Our paper extracts the kernel from the forest, and identifies the relevant tuning parameter.  We establish a model of the kernel which gives the user intuition for how changing this parameter affects the kernel's shape.  As with any kernel procedure in practice, the bandwidth parameter requires tuning: different sizes work better in different settings, completely dependent upon the problem at hand.

The practical implications of this paper are that for random forest probability estimation, tuning matters.  This point is easy to overlook because random forests misclassification error rate tends to be very robust to parameter settings.  In practice, it is tempting to trust that a random forest that produces a low misclassification error rate will also produce reasonable probabilities, but the examples in this paper have illustrated that this need not be the case.  

Finally, the `kernel' view of random forest also suggests extensions to the way random forests are used for classification and probability estimation.  As argued in Section~\ref{sec:prox_probs}, one can view random forests as kernel regression with the proximity function.  The proximity function is just one possible measure of `closeness' that can be extracted from the forest.  As a simple example, instead of a simple binary measure of whether two training points occupy the same terminal node of a tree, one might instead take into consideration the number of tree splits separating these points.

\bibliographystyle{plainnat}\bibliography{refs}

%%%%%%%%%%%%%%%%%%%%%%%%%%%%%%%%%%%%%%%%%%%%%%%%%%
%                                                        Appendix
%%%%%%%%%%%%%%%%%%%%%%%%%%%%%%%%%%%%%%%%%%%%%%%%%%

\appendix

\section{Equivalence of Splitting Criteria}
\label{sec:splitting}

It is a trivial fact that the mean squared error and Gini splitting criteria are equivalent when used with a binary outcomes $y \in \{0,1\}$.  As a consequence, random classification and regression forests fir to binary data only differ in the default parameter settings and aggregation method across trees.  We will present a simple argument for this fact.

Suppose that a candidate split of some node results in left and right daughter nodes, denoted by $R_L$ and $R_R$, each of which contains $N_L$ and $N_R$ training points, respectively.  The fitness criteria for such a split according to mean squared error is
\begin{equation}
\label{eq:split_req}
\sum_{x \in R_L} \left( y_i - p_L \right)^2 + \sum_{x \in R_R} \left( y_i - p_R \right)^2 
\end{equation}
while the fitness for the Gini criteria is
\begin{equation}
\label{eq:split_class}
N_L p_L \left(1-p_L\right) + N_R p_R \left(1-p_R\right)
\end{equation}
where $p_L = \frac{1}{N_L}\sum_{x \in R_L} y_i$ and $p_R = \frac{1}{N_R}\sum_{x \in R_R} y_i$.  We can expand the squares in the mean squared error criteria in Equation~\ref{eq:split_req} as follows:
\begin{align*}
\sum_{x \in R_L} \left( y_i - p_L \right)^2 + \sum_{x \in R_R} \left( y_i - p_R \right)^2  & = N_L p_L - 2 N_L p_L^2 + N_L p_L^2 + N_R p_R - 2 N_R p_R^2 + N_R p_R^2 \\
&= N_L p_L - N_L p_L^2 + N_L p_L - N_L p_L^2 \\
& = N_L p_L \left(1-p_L\right) + N_R p_R \left(1-p_R\right).
\end{align*}

\subsection{Naive Approximation}
\label{sec:approximation}

Suppose a random forest is grown to a size of $M$ nodes.  Under the setting of Algorithm~\ref{algo:simple_rf_app}, we will derive an approximation to the proximity function $K(0, z)$, where $z \in [0,1]^p$.  Our argument will borrow substantially from \cite{breiman2000} and \cite{scornet2016c}.  Under completely random splitting (no notion of strong and weak variables), \cite{breiman2000} found an approximation to the proximity function of the form $K(0,x) = \exp\left( -\log M /p\sum^p_{i=1} x_i \right)$ (although it is worth noting that the argument presented there claimed to approximate $K(x,z)$ for all $x, z \in [0,1]$, but it contained an error).  
\begin{algorithm}[htp]
\caption{Simplified Random Forest Tree}
\begin{algorithmic}
\State 1. Specify the number of leafs $M$; initialize $\texttt{leafs} = \{\mathfrak{t}_{root}\}$.
\State 2. For $m = 1:M$:
\State \indent (a) Select a terminal node $\mathfrak{t} \in \texttt{leafs}$ uniformly at random.
\State\indent  (b) Split $\mathfrak{t}$ into daughter nodes $\mathfrak{t}_{L}$, $\mathfrak{t}_{R}$
\State\indent\indent (i) Choose $\mtry$ predictors at random $\mathcal{F} \subseteq \{1, \ldots, p\}$
\State\indent\indent (ii) Select split variable uniformly among the $S_{*}$ \\ \indent \indent \indent \hspace{5mm} signal variables in $\mathcal{F}$
\State\indent\indent (iii) Choose split point uniformly at random
\State \indent (c) Replace $\mathfrak{t}$ with $\mathfrak{t}_{L}$ and $\mathfrak{t}_{R}$ in \texttt{leafs}
\end{algorithmic}
\label{algo:simple_rf_app}
\end{algorithm}

Specifically, we will assert that 
\begin{equation}
K(0,z) \approx \exp{ \left\lbrace -\log{M}  \left(  p_{\mathcal{S}} \sum_{s \in \Strong} z_s +  p_{\mathcal{W}} \sum_{w \in \Weak} z_w  \right) \right\rbrace}
\end{equation}
where 
\begin{align*}
p_{\mathcal{S}} &=  \sum^{S \wedge \mtry}_{k=1} \frac{{S-1 \choose k-1} {W \choose \mtry - k+1} }{k {p \choose \mtry}} \\
p_{\mathcal{W}} &= \frac{1-S p_{\mathcal{S}}}{W}.
\end{align*}

First, we will calculate the probability $p_s$ that a given strong variable $x_{s}$ is selected at a given node when there are $S$ strong variables, $W$ weak variables, and $\mtry$ predictors are considered at a time.  Let $R$ denote the number of strong variables selected among the $\mtry$, and $Q$ denote an indicator for whether $x_{s}$ is selected .  We will compute $p_{\mathcal{S}}$ by conditioning on $R$:
\begin{align*}
p_{\mathcal{S}} &= \sum^{S \wedge \mtry}_{k=1}  p\left(Q=1, R=k\right) \\
&= \sum^{S \wedge \mtry}_{k=1} \frac{{S - 1\choose k-1} {W \choose \mtry - k+1} }{k {p \choose \mtry}}.
\end{align*}
It is then easy to see that the probability of selecting a given weak variable is just $\frac{1-Sp_s}{W}$.

Next, we will compute the probability that $0$ and $z$ are in the same terminal node given that there are $k_j$ total splits on coordinate $j = 1, \ldots, p$.  Let $c_1, \ldots, c_{k_j}$ denote the randomly chosen split points for coordinate $j$.  Then the probability that these split points do not separate $0$ and $x_j$, $p\left(c_1\not\in [0, x_j], c_2\not\in [0, x_j] \ldots, c_{k_j} \not\in [0, x_j]\right)$,  is
\begin{equation}
\label{eq:integral}
\int_{c_1\not\in [0, x_j]} \int_{c_2\not\in [0, x_j]} \cdots \int_{c_{k_j} \not\in [0, x_j]}  p(dc_{k_j} | c_{k_j-1}) \cdots p(dc_2 | c_{1}) p(dc_1).
\end{equation}

Now, given $c_{k-1}$, the distribution of $c_{k}$ is $c_{k} | c_{k-1} \sim \mathcal{U}[0, c_{k-1}]$.  Thus, it holds that
\begin{align*}
p\parens{c_k\not\in [0, x_j] | c_{k-1}} = 1 - x_j/c_{k-1}.
\end{align*}
  One can then prove inductively that the integral in ~\ref{eq:integral} reduces to $1- x_j \sum^{k_j - 1}_{i=0} \frac{\parens{-\log{x_j}}^{i}}{i!}$.  If we let $w_j = -\log{x_j}$, this is precisely the probability that a Poisson random variable $Z_j$ with parameter $w_j$ is greater than $k_j$.  Next, we need to consider $p\parens{k_1, \ldots, k_p}$, the joint probability of cutting $k_j$ times on predictor $x_j$ for $j = 1, \ldots, p$.  If we condition on the total number of cuts $K$, then $p\parens{k_1, \ldots, k_p | K}$ is multinomial with $K$ total trials and success probabilities $\parens{p_s, \ldots, p_s, p_w, \ldots, p_w}$.  Finally, the distribution of $K$ is the sum of $M-1$ Bernoulli random variables, the $m$ of which has success probability $1/m$.  Putting this all together, 
\begin{equation}
\label{eq:exact}
K(0, z) = \sum^{T}_{k=1} p(k)  \sum^{k}_{k_1 + \cdots k_p = k} p\parens{k_1, \ldots, k_p | K = k} \prod^p_{j=1} p\parens{Z_j \geq k_j}.
\end{equation}
We seek a more tractable approximation to the probability computed in \ref{eq:exact}.  Following \cite{breiman2000}, we appeal to a Poisson approximation.  Let us assume that $p > 5$, $T \leq \exp\parens{p/2}$ and $W$ and $S$ are such that $p_s$ and $w_s$ are both small.  If there are currently $K$ nodes, then the probability of selecting the node  that $0$ and $z$ are both in is $1/K$.  So, $k_j$ is achieved by $T-1$ binomial trials, such that the probability of each is $p_s$ is $x_j$ is strong, and $p_w$ if $x_j$ is weak.  Thus, we will assume that $k_1, \ldots, k_p$ are independent Poisson random variables with parameter 
\begin{align*}
\lambda_s &= p_s \sum^{T-1}_{k=1} \frac{1}{k} \approx p_s \log{T} \\
\lambda_w &= p_w \sum^{T-1}_{k=1} \frac{1}{k} \approx p_w \log{T} 
\end{align*}
depending upon whether the predictor is strong or weak.  We may then approximate \ref{eq:exact} by
\begin{equation}
\prod_{s \in \mathcal{S}} U\parens{\lambda_s, w_s} \prod_{w \in \mathcal{W}} U\parens{\lambda_w, w_w}
\end{equation}
where $U\parens{\lambda_s, w_s} = e^{-\lambda_s - w_s} \sum^{\infty}_{k=1} \lambda_s^k p\parens{Z_s \geq w_s}$, and analogously for $U\parens{\lambda_w, w_w}$.  Finally, through a Laplace transform argument, \cite{breiman2000} argued that  $U\parens{\lambda_s, w_s} \approx e^{-\lambda_s w_s}$.

\section{Data Set Descriptions}
\label{sec:data_description}

\begin{table}[ht]
\label{table:dset_table}
\centering
\begin{tabular}{cccc|}
\hline 
Data Set & N & Features \\ 
\hline 
\textit{Australian credit} & 690 & 15 \\  
\textit{ionosphere} & 351 & 34 \\  
\textit{parkinsons} & 195 & 22\\  
\textit{spam} & 4601 & 57 \\  
\textit{splice} & 2422 & 60\\  
\textit{tic} & 958 & 9 \\  
\textit{voting} & 435 & 16 \\ 
\hline 
\end{tabular} 
\caption{Descriptions of UCI Repository data sets used in Section~\ref{sec:prob_comp}.}
\end{table}

\begin{itemize}

\item 1d Model
\begin{equation*}
\condprobyx =
\begin{cases}
0.3     & x_1 x_2 \geq 0\\
0.7 & x_1 x_2 < 0
\end{cases}
\end{equation*}

\item Circle Model
\begin{equation*}
\condprobyx =
\begin{cases}
1 , & ||x||_2 \leq  8\\
\frac{28 - ||x||_2}{20} , & 8 \leq ||x||_2 \leq 28\\
0,  & \text{otherwise}.
\end{cases}
\end{equation*}

\item Logistic Model
\begin{equation*}
\condprobyx = \frac{1}{1 + e^{-2 (x_1 + x_2 + x_3)}}.
\end{equation*}

\item 10d Model
\begin{equation*}
\log\left(\frac{\condprobyx)}{1-\condprobyx}\right) = 0.5 (1-x_1 + x_2 - \cdots +x_6) (x_1 + \cdots + x_6)
\end{equation*}

\item XOR Model
\begin{equation*}
\condprobyx = 
\begin{cases}
0.3, & x_1 x_2 \geq 0\\
0.7,  & x_1 x_2 < 0.                                           
\end{cases}
\end{equation*}
\end{itemize}

\end{document}